\documentclass[12pt]{msml2022} 
\usepackage[utf8]{inputenc}

\title{On the Nash equilibrium of moment-matching GANs for stationary Gaussian processes}

\usepackage{times}

\usepackage{tabularx} 
\usepackage{multirow}
\usepackage{booktabs}

\newcommand{\N}{\mathbb{N}}
\newcommand{\E}{\mathbb{E}}
\newcommand{\R}{\mathbb{R}}
\newcommand{\Z}{\mathbb{Z}}
\newcommand{\C}{\mathbb{C}}
\newcommand{\lb}{\langle}
\newcommand{\rb}{\rangle}
\newcommand{\om}{\omega}

\newcommand{\AL}{\mathcal{A}}
\newcommand{\BE}{\mathcal{B}}
\newcommand{\Zb}{\bar{Z}}

\newcommand{\Om}{\Omega}

\newcommand{\dotr}[1]{#1^{\bullet}} 
\newcommand{\what}{\widehat}
\newcommand{\al}{\alpha}
\newcommand{\alb}{\bar{\alpha}}
\newcommand{\ald}{\dotr{\alpha}}

\newcommand{\be}{\beta}

\newcommand{\bed}{\dotr{\beta}}

\newcommand{\la}{\lambda}
\newcommand{\lad}{\dotr{\lambda}}

\newtheorem{prop}{Proposition}
\newtheorem{defi}{Definition}
\newtheorem{assum}{Assumption}
\newtheorem{lemm}{Lemma}
\newtheorem{corr}{Corollary}

\usepackage{color}
\usepackage{xcolor}

\msmlauthor{%
 \Name{Sixin Zhang} \Email{sixin.zhang@irit.fr}\\
 \addr Universit\'e de Toulouse, INP, IRIT, Toulouse, France \\
 \\
 \textbf{Editors: Bin Dong, Qianxiao Li, Lei Wang, Zhi-Qin John Xu} 
}

\makeatletter
 \let\Ginclude@graphics\@org@Ginclude@graphics
\makeatother

\begin{document}

\maketitle

\begin{abstract}%

Generative Adversarial Networks (GANs) 
learn an implicit generative model 
from data samples through a two-player game. 
In this paper, we study the existence of Nash equilibrium 
of the game which is consistent as the number of data samples grows to infinity.
In a realizable setting where the goal is to estimate
the ground-truth generator of a stationary Gaussian process, 
we show that the existence of consistent Nash equilibrium 
depends crucially on the choice of the discriminator family. 
The discriminator defined from 
second-order statistical moments 
can result in non-existence of Nash equilibrium, 
existence of consistent non-Nash equilibrium, 
or existence and uniqueness of consistent Nash equilibrium,
depending on whether symmetry properties of the generator family
are respected. We further study empirically
the local stability and global convergence of gradient descent-ascent 
methods towards consistent equilibrium.

\end{abstract}



\section{Introduction}

Estimating the probability distribution of data from finite samples is a classical 
problem in statistics and machine learning. 
Unlike conventional models based on a probability density function, 
Generative Adversarial Networks (GANs)
aim to learn a generator, 
which describes how to draw samples of model distributions,
through a two-player game with a discriminator. 
A central question in GANs is what type 
of solutions of the game is suitable 
for learning the data distribution.
The original formulation \citep{goodfellow2014generative}
shows that Nash equilibrium exists in an ideal realizable setting, 
where infinite data samples can be generated from a ground-truth generator.
However, in some unrealizable settings,
\cite{farnia20a} show that GANs may have no Nash equilibrium. 
Indeed, many research works are devoted to 
extending the notion of Nash equilibrium to 
other types of solutions such as mixed Nash equilibrium \citep{pmlr-v70-arora17a},
proximal equilibrium \citep{farnia20a}, or local minimax \citep{jin2020local}.
Nevertheless, in practice, GANs are often trained with gradient-based methods
\citep{mescheder2017numerics,brock2018large}, and it is shown 
that Nash equilibrium is typically contained in 
the limiting points of such methods \citep{Daskalakis18}. 
This motivates one to study in what situations GANs have meaningful Nash equilibrium,
and whether it is possible to find such solution using gradient-based methods.

To make this problem concrete, 
we shall focus on a special instance of GANs in the realizable setting with finite-samples.
In particular, we study a notion of consistent Nash equilibrium, 
which allows one to estimate the ground-truth generator
as the number of samples grows to infinity. 
The main contribution of this paper is to show that 
the existence of consistent Nash equilibrium 
depends crucially on the choice of the discriminator family. 
More precisely, 
in order to learn the ground-truth generator of a stationary Gaussian distribution,
we find that a suitable discriminator family based on second-order moments
which respects the symmetry property 
of stationary processes, i.e. the translational invariance, 
can result in the existence and the uniqueness of consistent Nash equilibrium.
When this is not the case, Nash equilibrium may not exist
due to a finite number of samples. 
More surprisingly, we find that there can exist consistent 
non-Nash equilibrium around which 
gradient-descent-ascent methods are nearly stable 
as the number of samples goes to infinity. 
This indicates the possibility that 
GANs training may practically converge
to non-Nash equilibrium
while still achieving good results,
as observed in \cite{Berard2020A}.
To understand why this may happen in practice remains an interesting open problem. 

This paper is organized as follows:
Section \ref{sec:pre} introduces Moment-matching GANs 
for stationary Gaussian distributions, 
as well as the notion of consistent Nash equilibrium. 
The generator is parameterized by a linear convolutional network. 
To define the discriminator, we take a moment-matching perspective 
as in MMD GANs \citep{pmlr-v37-li15,uaiMMDgan}. 
Section \ref{sec:cne} presents the main theoretical results 
about the existence of consistent Nash equilibrium
on three different families of discriminators. 
Section \ref{sec:num} 
studies numerically how gradient-based methods 
for two-player games behave on the considered GANs from 
either local or global convergence point of view. 
Section \ref{sec:conclude} discusses
the challenges to extend our results to non-Gaussian distributions.

%
%
%

\textbf{Notations:}
We use $\| \cdot \|$ to denote the Euclidean metric in a finite dimensional space. 
The identify matrix in dimension $d$ is $I_d$. 
For $x \in \R^d$, the discrete Fourier transform of $x$ is written $\what{x}$. 
It is defined by $\what{x}(\om) = \sum_{u} x(u) e^{-i \om u}$ 
for $u \in \{ 0, \cdots, d-1\}$ and $\om \in  \Om_d = \{ \frac{2 \pi \ell } {d}  | 0 \leq \ell \leq d-1 \}$. 
For a complex number $z$, $z^\ast$ denotes its complex conjugate. 
For $x \in \C^d$, we write $\tilde{x}(u) = x(-u)^\ast$. 
The inner product between $x \in \C^d$ and $y \in \C^d$ is written $\lb x,y \rb = \sum_u x(u)^\ast y(u)$.

\section{Preliminaries}\label{sec:pre}

\subsection{Moment matching GANs}

We consider a class of GANs defined by a generator $ g_\al: \R^k \to \R^d$, 
and a discriminator $f_\be : \R^d  \to \R^m$. 
They are parameterized by $\al \in \AL$ and $\be \in \BE$ 
in finite dimensional Euclidean spaces.  
In this paper, we consider the realizable setting where
there is a ground-truth model 
$g_{\alb}$ which generates the data $X = g_{\alb} (\Zb)$
from a random vector $\Zb$ for some $\alb \in \AL$. 
Our goal is to find a solution 
$(g_{\ald},f_{\bed})$ of GANs such that
$g_{\ald}$ is close to $g_{\alb}$ in terms of their probability distributions.

A moment matching GAN can be formalized as the following min-max problem, 
\begin{equation}\label{eq:V}
	\min_{\al \in \AL} \max_{\be \in \BE} 
	\| \E ( f_\be( X) ) - \E ( f_\beta ( g_\al(Z) ) ) \|^2  .
\end{equation}
We shall consider the min-max problem with $n$ finite samples
\begin{equation}\label{eq:Vn}
	\min_{\al \in \AL} \max_{\be \in \BE} V_n ( \al , \be )  = \| \E_n ( f_\be ( X) ) - \E_n ( f_\be ( g_\al (Z) ) ) \|^2  . 
\end{equation}
The empirical expectation $\E_n(f(X))$ of a function $f$ of a random variable $X$ is computed 
from $n$ i.i.d samples of $X$. 
The metric $\| \cdot \|$ in \eqref{eq:V} and \eqref{eq:Vn} is chosen such that this is an MMD GAN in the Euclidean space $\R^m$. 
It is related to various GANs, 
from the perspective of feature matching \citep{liu2017approximation}. 

\subsection{Stationary Gaussian processes}

We specify the generator $g_\al$ to 
model stationary Gaussian processes
observed in a finite and discrete-time interval. 
Without the loss of generality, we consider circular stationary process defined by 
\[
	g_\al ( Z)   = \al \star Z, \quad \al \in \AL = \R^d . 
\]
It computes a circular convolution 
between a filter $\al$ and 
a Gaussian white noise $Z \sim \mathcal{N} (0,I_d)$ 
on the interval $\{0,\cdots, d-1 \}$. 
As $d$ is finite, $g_\al ( Z)$ can be regarded as a zero-mean linear stationary process
observed on this interval \citep{Priestley}. 

To measure the closeness between $g_\al(Z)$ and $X$ (two zero-mean Gaussian distributions),
we compute the spectral norm of the difference of their covariance matrices $\Sigma_{\al} = \E( g_\al ( Z) g_\al ( Z)^\intercal )$ and $\Sigma = \E( X X^\intercal)$. This allows one to define an error to measure the quality of a generator. By definition, $X = g_{\alb} ( \Zb)$ is generated by a Gaussian white noise $\Zb$, which is independent of $Z$. 
Due to the stationarity of $g_\al(Z)$,
the spectral norm can be computed in the Fourier domain by the following classical result. 
\begin{prop}\label{propSigmaFourier}(Generator error)
	For any $\al \in \AL$, we have
	 \[
	 	\| \Sigma_{\al}  - \Sigma \| = 
		\max_{\om \in \Om_d} |  |\what{\al}  (\om) |^2 - | \what{\alb} (\om) |^2  |  . 
	 \]
\end{prop}
\begin{proof}
 For any $\al \in \AL$, $\Sigma_{\al} $ is a Toeplitz and circulant matrix, thus it can be diagonalized by the discrete Fourier transform. As $Z$ is a stationary Gaussian white noise, the eigenvalues of $\Sigma_{\al} $ are given by $|\what{\al}  (\om) |^2$ for $\om \in \Om_d$ \citep{Priestley}. 
	The spectral norm is thus the maximal absolute difference between the eigenvalues of $\Sigma_{\al} $ and $\Sigma=\Sigma_{\alb} $. 
 \end{proof}

In this paper, we are interested in a particular set of generators defined by
\[
	\AL_n =  \{  \al \in \R^d | | \what{\al}(\om) |^2 = \E_n (  |\what{X}(\om) |^2 )  /  \E_n (  |\what{Z}(\om) |^2 ), \forall \om \in \Om_d \} . 
\]
In fact, every sequence of generators in $\AL_n$ is consistent in the following sense. 
\begin{prop}\label{propConsistent}(Consistent generator)
	Assume $\al_n \in \AL_n$, then $\|  \Sigma_{\al_n}  - \Sigma \| \to 0 $ in probability as $n \to \infty$. 
\end{prop}
\begin{proof}
	The set $\AL_n$ is well defined, as almost surely $  \E_n (  |\what{Z}(\om) |^2 ) \neq 0$ for all $\om \in \Om_d$. The law of large numbers implies that for any $\om \in \Om_d$, 
	\[
	\E_n (  |\what{\Zb}(\om) |^2 ) / \E_n (  |\what{Z}(\om) |^2 )  \to 1  , \quad n \to \infty, \quad \mbox{in probability}. 
	\]
	As $\E_n (  |\what{X}(\om) |^2 )  =  \E_n (  |\what{\Zb} (\om) |^2 )   | \what{\alb}(\om) |^2 $, it implies the convergence of $| \what{\al_n}(\om) |^2 $ to $| \what{\alb}(\om) |^2 $ in probability
	for any $\om \in \Om_d$. From Proposition \ref{propSigmaFourier}, we conclude that  
	$\|  \Sigma_{\al_n}  - \Sigma \| \to 0 $ in probability. 
\end{proof}

For some technical reasons in Section \ref{sec:cne}, 
we next introduce a working assumption 
regarding the number of samples $n$, and the ground-truth generator  $g_{\alb} $. 
\begin{assum}\label{assumX}
Assume $n \geq 2$, $d $  is even, and $\alb \not \in \AL_0$, where 
\[
	\AL_0 =  \{ \al \in \AL | \exists \om  \in \Om_d, \what{\al} (\om) = 0\} .
\]
\end{assum}
The condition $\alb \not \in \AL_0$ is needed to avoid degeneracy (i.e. to avoid having zero eigenvalues in $\Sigma$). 
We are thus considering stationary processes whose power spectrum
are supported on the Fourier domain $\Om_d$.\footnote{The power spectrum of a stationary process $X$ observed over an interval of length $d$ is defined as the limit of $ \E ( | \what{X} (\om) |^2 ) /d $ as $d \to \infty$. In this paper, we use the same name for the finite $d$ case (without taking the limit).  
}

\subsection{Nash equilibrium and its consistency}

To study the solution
of moment matching GANs, 
we review the notion of Nash equilibrium in  
differentiable zero-sum games,
and then discuss its statistical consistency property. 
This property allows one to estimate the ground-truth generator $g_{\alb}$ as $n \to \infty$.

In this paper, 
we assume that 
both $\AL$ and $\BE$ belong to finite-dimensional Euclidean spaces, 
and that $V_n$ is everywhere twice-differentiable with respect to $\al$ and $\be$.

\begin{defi}(Equilibrium)
Let $\AL$ and $\BE$ be open sets. We say that $(\ald,\bed)$ is an 
equilibrium of a game $(\AL,\BE,V_n)$, or an equilibrium of $V_n$, if 
\begin{equation}\label{eq:grad0}
	\nabla_{\al} V_n  (\ald,\bed) = 0, \quad \nabla_{\be} V_n  (\ald,\bed) = 0. 
\end{equation}
\end{defi}

This notion of equilibrium is used in 
\cite{Daskalakis2019} to study fixed points of games. 
It is also called stationary or critical point \citep{Daskalakis18, jin2020local}. 

\begin{defi}(Nash equilibrium)
We say that $(\ald,\bed)$ is a Nash equilibrium of $(\AL,\BE,V_n)$, or a Nash equilibrium of $V_n$, if 
\begin{equation}\label{eq:nash}
	V_n ( \ald, \be) \leq 
	V_n ( \ald, \bed) \leq 
	V_n ( \al, \bed) , 
	\quad  \forall  \al \in \AL, \forall  \be \in \BE. 
\end{equation}
For open $\AL$ and $\BE$, we say $(\ald,\bed)$ is a 
non-Nash equilibrium when \eqref{eq:grad0} holds and \eqref{eq:nash} does not hold, 
\end{defi}

Note that when $\AL$ and $\BE$ are open sets, a Nash equilibrium $(\ald,\bed)$ of $V_n$ is also an equilibrium because \eqref{eq:nash} implies \eqref{eq:grad0}. 
As the original GANs \citep{goodfellow2014generative}, 
Nash equilibrium always exists for the game defined in \eqref{eq:V}.
The existence is less clear for $V_n$ as we have only finite samples:
we are no longer in the situation where $(\alb,\be)$ is a Nash equilibrium for any $\be \in \BE$. 
This is different to the unrealizable setting in \cite{farnia20a} 
which assumes in our context that $\alb \not \in \AL$ and studies \eqref{eq:V} rather than $V_n$. 


In finite-sample realizable GANs, 
we can introduce a notion of consistent equilibrium 
as in the classical estimation theory in statistics \citep{ibraHas1981}. 
This can be formalized in our context 
using the following definition. 

\begin{defi}(Consistent Nash equilibrium)
We say that $\{(\ald_n,\bed_n) \}_{n \geq 1}$ 
is a sequence of consistent Nash equilibrium if $\exists N \in \N$ such that 
$\forall n \geq N$, $(\ald_n,\bed_n)$ is almost surely a Nash equilibrium of $V_n$,
and it satisfies
\begin{equation}\label{eq:cst}
	\| \Sigma_{\ald_n}  - \Sigma \|  \to 0, \quad  n \to \infty
	\quad \mbox{in probability} . 
\end{equation}
When $\AL$ and $\BE$ are open sets, we say that
$ \{ (\ald_n,\bed_n) \}_{ n \geq 1}$ 
is a sequence of consistent equilibrium if $\exists N \in \N$ such that $\forall n \geq N$,
\eqref{eq:grad0} holds for $(\ald_n,\bed_n)$ and \eqref{eq:cst} holds.
It is a sequence of consistent non-Nash equilibrium 
if it is a sequence of consistent equilibrium, and 
$\exists N \in \N$ such that $\forall n \geq N$,
\eqref{eq:nash} does not hold almost surely for $(\ald_n,\bed_n)$. 
\end{defi}

\begin{remark}
In the next, we say simply 
$(\ald_n,\bed_n)$ is a consistent Nash (resp. non-Nash) equilibrium of $V_n$
without mentioning the sequence.
Note that in the above definition, we do not specify whether $\bed_n$ 
is a convergent sequence, because what matters most in our problem 
is the convergence of $g_{\ald_n}$. 
The convergence in terms of the covariance matrix in \eqref{eq:cst} may be generalized 
to measure certain distance or divergence between probability distributions. 
\end{remark}

In some situations, we are interested in 
consistent Nash equilibrium or consistent
non-Nash equilibrium. 
We next review a necessary condition of Nash equilibrium
obtained from the Jacobian matrix of a differentiable (continuous) game.

\begin{defi}(Jacobian matrix)
Let $\AL$ and $\BE$ be open sets. 
The gradient vector of $V_n$ is 
\[
  \left( \begin{array}{c}
\nabla_{\al} V_n (\al,\be)    \\
-\nabla_{\be} V_n (\al,\be) 
\end{array} \right). 
\]
The Jacobian matrix of $V_n$ is the derivative of the gradient vector, i.e. 
\[
J_n (\alpha,\beta) = \left( \begin{array}{cc}
\nabla^2_{\al\al} V_n (\al,\be)  & \nabla^2_{\al\be} V_n (\al,\be)   \\
- \nabla^2_{\be\al} V_n (\al,\be)  & - \nabla^2_{\be\be} V_n (\al,\be)  
 \end{array} \right)  . 
\]
\end{defi}

The following result is known in differentiable games \cite[Proposition 2]{Ratliff13}.
\begin{prop}\label{propNashsemi}
Let $\AL$ and $\BE$ be open sets. If $(\ald,\bed)$ is a Nash equilibrium of $V_n$, then both $\nabla^2_{\al\al} V_n (\ald,\bed) $ and $-\nabla^2_{\be\be} V_n (\ald,\bed)  $ are semi-positive definite. 
\end{prop}

\begin{remark}\label{rmq:nonNash}
It follows that if $ \nabla^2_{\be\be} V_n (\ald,\bed) $ has at least one strictly positive eigenvalue, then the equilibrium $(\ald,\bed)$ is a non-Nash equilibrium. 
\end{remark}

\section{Existence of consistent Nash equilibrium}
\label{sec:cne}

In this section,
we study the impact of the discriminator family $\{ f_\be, \be \in \BE \}$ 
on the existence of consistent Nash equilibrium of $V_n$.
As we consider only Gaussian stationary processes, 
all the discriminators are constructed from second-order statistical moments.
For non-Gaussian distributions such as those 
generated by a one-layer neural network, 
second-order moments are also 
used to construct the discriminator family \citep[Section 5]{lei2020sgd}.
However, the existence of Nash equilibrium has not been studied in these GANs. 

\subsection{Real discriminator}

Consider
\[
	f_\be (X) = |\lb \be, X \rb|^2 , \quad \BE = \{ \be \in \R^d | \| \be \| \leq 1 \} .
\]

This discriminator has only one feature ($m=1$), and it is called a real discriminator 
because $\be$ is a real-valued vector. The next result shows that it can completely capture the spectral properties of $X$ due to the maximization of $V_n$ with respect to $\be $. 

\begin{prop}\label{prop1}
	Let $\Sigma_n = \E_n( X X^\intercal )$, $\Sigma_{\al,n} = \E_n( g_\al (Z) g_\al (Z)^\intercal ) $, 
	then for any $\al \in \AL$, 
	\[
		\sup_{\be \in \BE } V_n ( \al , \be ) = \| \Sigma_n - \Sigma_{\al,n } \|^2   .
	\]
	Moreover, if $\Sigma_n \neq \Sigma_{\al,n}$, 
	then the optimal $\be$ is a unit-norm 
	eigenvector of $\Sigma_n - \Sigma_{\al,n }$ 
	which has the largest absolute eigenvalue.
\end{prop}

The proof is given in Appendix \ref{proofkkt}. 
We next show that in general, 
there is no generator which can achieve a zero error of $ \| \Sigma_n - \Sigma_{\al,n } \|^2  $, 
i.e. the empirical covariance of $X$ and $g_\al(Z)$ can not be perfectly matched.

\begin{lemm}\label{lem:nonzero}
Under Assumption \ref{assumX}, $\forall \al \in \AL$, we have almost surely
$ \| \Sigma_n - \Sigma_{\al,n} \| 	 > 0 $. 
\end{lemm}

The proof is given in Appendix \ref{proofnonzero}. 
It is due to the fact that $n$ is finite, and the samples of $X=g_{\alb}(\Zb)$ are generated from 
$\Zb$, which are independent of the samples from $Z$. 
Based on this result,
we next show that there is no Nash equilibrium in $V_n$.

\paragraph{Non-existence of Nash equilibrium}

Assume that $(\ald,\bed)$ 
is a Nash equilibrium of $V_n$, 
then it is a best response solution
for each player, i.e. 
\begin{align}
\ald & \in  \arg \min_{ \al \in \AL }  V_n ( \al,\bed )  \label{eq1} \\
\bed & \in \arg \max_{ \be \in \BE }  V_n ( \ald,\be )  \label{eq2}
\end{align}

The next result shows that such solution does not exist in general. It implies that no consistent Nash equilibrium exists in $V_n$. 
\begin{theorem}\label{ref:realNonNE}
	Under Assumption \ref{assumX}, there is almost surely no Nash equilibrium in $V_n$. 
\end{theorem}
\begin{proof}
	From \eqref{eq2} and Proposition \ref{prop1}, it follows that $  V_n ( \ald,\bed )  = \| \Sigma_n - \Sigma_{\ald,n} \|^2$. We next show that \eqref{eq1} does not hold, i.e. 
	there exists $\al \in \AL$ such that 
	\[
	V_n ( \al, \bed) = \lb \bed , ( \Sigma_n - \Sigma_{\al,n}) \bed \rb < V_n ( \ald,\bed ) . 
	\]
		
	We minimize $V_n (\al,\bed)$ with respect to $\al$. Note that 
	\begin{align}\label{eq:vn1}
		V_n (\al,\bed) &= (  (\bed)^\intercal ( \Sigma_n  - \Sigma_{\al,n} )  \bed )^2  \nonumber \\
			&= (  (\bed)^\intercal \Sigma_n  \bed )^2 + 
				(  (\bed)^\intercal \Sigma_{\al,n}  \bed )^2 - 
				2 (  (\bed)^\intercal \Sigma_n  \bed ) 	(  (\bed)^\intercal \Sigma_{\al,n}  \bed ) . 
	\end{align}	
	Assume $\{ z_i \}$ are the $n$ i.i.d. samples to compute the empirical expectation of $Z$ in $V_n$. 
	We rewrite $	V_n (\al,\bed)$ more explicitly in terms of $\al$, in the following equation
	\begin{equation}\label{eq:vn2}
		(\bed)^\intercal \Sigma_{\al,n}  \bed =  \E_n  ( |  \lb \bed,  \al \star Z \rb |^2 )  =  \al^\intercal S_{\bed} \al  , 
	\end{equation}
	where $S_{\bed}(v,v') = \frac{1}{n} \sum_{i=1}^n \bed \star \tilde{z}_i (v)  \bed \star \tilde{z}_i (v')  $
	for $(v,v') \in \{ 0,\cdots, d-1 \}^2$. 
	
	It follows from \eqref{eq:vn2} and \eqref{eq:vn1} that 
	\begin{equation}
		\nabla_{\al} V_n (\al,\bed) = 4	(  (\bed)^\intercal \Sigma_{\al,n}  \bed )  S_{\bed} \al   - 4 (  (\bed)^\intercal \Sigma_n  \bed )   S_{\bed} \al . 		
	\end{equation}
	
	Setting $\nabla_{\al} V_n (\al,\bed) = 0$ implies two situations
	\begin{itemize}
		\item $ S_{\bed} \al = 0$: from \eqref{eq:vn1} and \eqref{eq:vn2}, 
		it follows that $V_n (\al,\bed) =  (  (\bed)^\intercal \Sigma_n  \bed )^2 $. 
		\item  $ (\bed)^\intercal \Sigma_{\al,n}  \bed = (\bed)^\intercal \Sigma_n  \bed  $: from \eqref{eq:vn1}, it follows that $V_n (\al,\bed) = 0 $.
	\end{itemize}
	The second situation implies that the minimum of $V_n(\al,\bed) $ is zero, as long as one can find a solution $\al$ such that 
	\begin{equation}\label{eq:grad2}
		 (\bed)^\intercal \Sigma_{\al,n}  \bed = (\bed)^\intercal \Sigma_n  \bed  . 
	\end{equation}
	We next show that almost surely, this is possible, i.e. $\min_{ \al \in \AL } V_n (\al, \bed) = 0$. 
	However, this contradicts to Lemma \ref{lem:nonzero}, which implies that almost surely $\Sigma_n - \Sigma_{\ald,n} \neq 0$, i.e. $V_n( \ald,\bed) = \| \Sigma_n - \Sigma_{\ald,n} \|^2 > 0$. Thus $(\ald,\bed)$ can not be a Nash equilibrium. 
	
	From \eqref{eq:vn2} and \eqref{eq:grad2}, 
	it remains to find $\al \in \AL$ such that 
	\begin{equation}\label{eq:aleq}
		\al^\intercal S_{\bed} \al   = (\bed)^\intercal \Sigma_n  \bed  . 
	\end{equation}
	
	From Proposition \ref{prop1}, we know that $\bed$ is a unit-norm eigenvector of 
	$ \Sigma_n - \Sigma_{\ald,n}$. As a consequence, almost surely
	\begin{equation}\label{eq:nonzero2}
		\forall i \leq n, \quad \bed \star \tilde{z}_i \neq 0 . 
	\end{equation}
	This is because $\tilde{z}_i $ is sampled i.i.d. from $\mathcal{N}(0,I_d)$, thus almost surely $\what{\tilde{z}}_i (\om) \neq 0 $ for all $\om \in \Om_d$. As $\what{\bed} \neq 0$, it follows that $ \what{\bed} \what{ \tilde{z}_i } \neq 0 $. This proves our claim of \eqref{eq:nonzero2}. 
	
	By definition, $ S_{\bed}$ is a semi-definite positive matrix. It results from \eqref{eq:nonzero2} that $ S_{\bed}$ has at least one strictly positive eigenvalue (otherwise $ S_{\bed}=0$). Let $h$ be one of the eigenvectors of $ S_{\bed}$ 
	whose eigenvalue $\la > 0$. Assume $\| h \|=1$. To construct a minimal solution of $V_n(\al,\bed)$, it suffices to take 
	\[
		\al = c h ,\quad 
		c = \sqrt{  \frac{ (\bed)^\intercal \Sigma_n \bed }  {   \la } }  .
	\]
\end{proof}

\begin{remark}\label{rmq:RealDisc}
This result suggests that although $\bed$ is 
an optimal discriminator for the generator $g_{\ald}$, 
it is not able to stabilize the minimization process of $\al$
because $ \| \Sigma_n - \Sigma_{\al,n } \|^2  $ is non-zero. 
This phenomenon is apparently due to the finite sample size $n$, but
it is also related to the fact that the generator family 
is restricted to be stationary. 
A larger generator family which can achieve a zero generator error 
may remedy the issue of the non-existence of Nash equilibrium. 
However, such model is necessarily non-stationary 
since $\Sigma_n$ is not a Toeplitz matrix. 
\end{remark}

Although there is no Nash equilibrium in $V_n$, 
one can still solve the min-max problem \eqref{eq:Vn} from a sequential game point of view.
For each $\al$, it first solves the maximization problem of $\be$, 
and then modifies $\al$ in order to minimize $ \max_{\be} V_n (\al,\be)$. 
There is a recent trend in the literature 
to extend the notion of Nash equilibrium towards 
the notion of Stackelberg equilibrium for sequential games, 
see e.g. \cite{jin2020local, fiez2020gradient, farnia20a}. 
However, the remaining optimization problem is challenging in practice,
in terms of the landscape of solution sets \citep{SunNeurips2020}, 
as well as the algorithm design and convergence analysis \citep{Wang2020ICLR}. 
In this paper, we take the simultaneous game point of view 
where the notion of Nash equilibrium is fundamental \citep{game2019}. 
We next study the existence of consistent Nash equilibrium in $V_n$
by looking for translational invariant features in the discriminator.

\subsection{Complex discriminator}

Consider 
\[
f_\be (X) = ( |\lb \be_\ell , X \rb|^2 )_{0 \leq \ell < m} , 
\quad \BE = \{   (\be_\ell )_{ 0 \leq \ell < m} |  \be_\ell \in \C^d  \} . 
\]

According to the proof of Proposition \ref{propSigmaFourier}, 
the discrete Fourier basis diagonalizes the 
covariance matrice $\Sigma_\al$ of $g_\al(Z)$ for any $\al \in \AL$.
This family of discriminator is constructed 
to contain this basis.
We show that when using $m=d$ features, 
consistent non-Nash equilibrium exists, 
and it is getting closer to a consistent Nash equilibrium as $n \to \infty$.

\paragraph{Existence of consistent non-Nash equilibrium}
Proposition \ref{propConsistent} shows that $ \AL_n$ is a set of consistent generators. 
We construct a family of consistent equilibrium $(\ald,\bed)$ where $\ald \in \AL_n$ and $\bed \in \BE$.
Note that in the definition of $\BE$, 
we do not impose any norm constraint on each $\be_\ell$. 
This simplifies our analysis of the equilibrium.

\begin{theorem}\label{thm:FouriernonNash}
	Assume $m = d$, $\bed_\ell (u) = e^{i 2 \pi \ell u / d } $ for $ 0 \leq u  < d$ and $0 \leq \ell < d$, then
	$\forall \ald \in \AL_n$, we have that $(\ald,\bed)$ is a consistent equilibrium of $V_n$ such that
	\[
		V_n ( \ald, \bed ) = 0 .
	\]
	Moreover, under Assumption \ref{assumX}, 
	$( \ald, \bed )$ is a consistent non-Nash equilibrium. 
\end{theorem}

\begin{proof}
We write $V_n (\al,\be) = \sum_{\ell  < m} r_{n,\ell}^2 (\al,\be)$, where
\[
	r_{n,\ell}  (\al,\be) = \E_n ( |\lb \be_\ell , X \rb|^2   ) - \E_n ( |\lb \be_\ell , g_\al(Z) \rb|^2   ) 
\]

To show that  $( \ald, \bed )$  is an equilibrium of $V_n$, we evaluate the gradient vector of $V_n$, which is computed with
\begin{equation}\label{eq:gradsVr}
\nabla_\al  V_n  = 2 \sum_\ell r_{n,\ell}  \nabla_\al  r_{n,\ell}   , \quad
\nabla_\be  V_n  = 2 \sum_\ell r_{n,\ell}  \nabla_\be  r_{n,\ell} .
\end{equation}

We verify that $\forall \ald \in \AL_n$ and $\forall \ell < d$, $r_{n,\ell} (\ald,\bed) = 0$. 
Indeed, by the definition of $\bed$, $\lb \bed_\ell , X \rb$ is the Fourier transform 
of $X$ at the frequency $\om = 2 \pi \ell / d \in \Om_d$. Thus $\ald \in \AL_n$ implies that 
\begin{equation}\label{eq:perfectmatching}
	\forall \om \in \Om_d, \quad 
	\E_n | \what{X} (\om) |^2 = \E_n | \what{g_{\ald} (Z)} (\om) |^2 = | \what{\ald} (\om)|^2 \E_n ( | \what{Z} (\om) |^2 )
\end{equation}
This is equivalent to $r_{n,\ell} (\ald,\bed) =0$ for all $\ell $. From \eqref{eq:gradsVr},
\eqref{eq:perfectmatching} and Proposition \ref{propConsistent}, 
$( \ald, \bed )$  is a consistent equilibrium of $V_n$. 

To show that it is not a Nash equilibrium, it is sufficient to verify that the Jabocian matrix at this equilibrium has a non-positive definite symmetric part, according to Remark \ref{rmq:nonNash}. Indeed, From \eqref{eq:gradsVr} and \eqref{eq:perfectmatching}, the symmetric part of $	J_n (\ald,\bed)$ is
\[
	\left ( 
\begin{array}{cc} 		
2 \sum_\ell (\nabla_\al r_{n,\ell} ) ( \nabla_\al  r_{n,\ell} )^\intercal  & 0  \\ 
 0  & 
- 2 \sum_\ell (\nabla_\be r_{n,\ell} ) ( \nabla_\be  r_{n,\ell} )^\intercal 	
\end{array} 
\right ) . 
\]

We next check that, under Assumption \ref{assumX}, 
there exists almost surely at least one $\ell \in \{0 ,\cdots , d-1 \}$ such that $\nabla_\be r_{n,\ell}  \neq 0$. 
Let $\be_\ell= \be_\ell^{re} + i \be_\ell^{im}$, 
we verify that at the equilibrium $ (\ald,\bed)$, 
\begin{equation}\label{eq:gradsReim}
	\nabla_{ \be_{\ell}^{re}}  r_{n,\ell}    =  
	2  ( \Sigma_n - \Sigma_{\ald,n })  (\bed_\ell )^{re}  ,\quad 
	\nabla_{ \be_{\ell}^{im}}  r_{n,\ell}    =  
	2  ( \Sigma_n - \Sigma_{\ald,n })  (\bed_\ell )^{im} 
\end{equation}
For $\ell' \neq \ell$, by the definition of $r_{n,\ell}$ we have 
\[
	\nabla_{ \be_{\ell'}^{re}}  r_{n,\ell}    =  0  ,\quad 
	\nabla_{ \be_{\ell'}^{im}}  r_{n,\ell}    =  0. 
\]

As $\nabla_\be r_{n,\ell} $ is a vector which concatenates
$ \nabla_{\be_{\ell'}} r_{n,\ell}  = ( \nabla_{ \be_{\ell'}^{re}}  r_{n,\ell}  , \nabla_{ \be_{\ell'}^{im}}  r_{n,\ell}  ) $
for $\ell' < d$, 
we conclude that $ \sum_\ell (\nabla_\be r_{n,\ell} ) ( \nabla_\be  r_{n,\ell} )^\intercal $ 
is a block diagonal matrix with $d$ blocks. 
The $\ell$-th block equals to a rank-one matrix
\[
	(\nabla_{\be_{\ell}} r_{n,\ell} )  ( \nabla_{\be_{\ell}} r_{n,\ell} )^\intercal . 
\] 

We next claim that at least for one $\ell$, $ \nabla_{\be_{\ell}} r_{n,\ell}$ is non-zero. 
Otherwise, setting the gradients in \eqref{eq:gradsReim} to zero for all $\ell$ implies that 
\[
	\Sigma_n - \Sigma_{ n,\ald } = 0 , 
\]
since $(\bed_\ell)_{\ell < d}$ forms an orthogonal basis in $\C^d$. However, this is contradictory to Lemma \ref{lem:nonzero} that, under Assumption \ref{assumX}, we have almost surely $\| \Sigma_n - \Sigma_{ n,\ald } \|  > 0 $. We conclude that the symmetric part of $	J_n(\ald,\bed)$ is not semi-positive definite,
and $( \ald, \bed )$ is a consistent non-Nash equilibrium. 
	
\end{proof}

\begin{remark}\label{rmq:uniform}
This result shows that $g_{\ald}(Z)$ matches empirically the power spectrum of $X$
because of the perfect moment matching: $V_n (\ald,\bed)= 0$. However,
the fact that $ \Sigma_n \neq \Sigma_{\ald ,n}$ still
does not allow one to obtain a Nash equilibrium
due to the loss of the semi-positive definiteness in the Jacobian matrix of $V_n$. 
Nevertheless, one can show that (a proof is given in Appendix \ref{app:uniform}), 
\begin{equation}\label{eq:asympSigma}
\sup_{  \ald \in \AL_n  }  \| 	\Sigma_n - \Sigma_{ \ald, n }    \|  \to  0 , \quad n \to \infty, \quad \mbox{in probability} . 
\end{equation}
From \eqref{eq:gradsReim}, \eqref{eq:asympSigma} we deduce that when $n$ is large, 
the symmetric part of the Jacobian matrix of $V_n$ at the equilibrium $(\ald,\bed)$
has vanishing negative eigenvalues. 
Therefore, the non-Nash equilibrium becomes closer to a Nash equilibrium as $n$ grows. 
\end{remark}

\subsection{Convolutional discriminator}

Consider 
\[
f_\be (X) = ( \|  X \star \be_\ell  \|^2 )_{0 \leq < \leq m} , 
\quad \BE = \{ (\be_\ell)_{ 0  \leq \ell < m }  | \be_\ell \in \R^d \} . 
\]
Unlike the real and complex discriminator, 
this discriminator family uses features that are always 
invariant to the translations
of $X$ on the grid of $u$.\footnote{The translations are defined with periodic boundary conditions.} 
Under appropriate assumptions, 
we show that $V_n$ admits infinite many consistent Nash equilibria, 
and in some sense they are unique.

\paragraph{Existence of consistent Nash equilibrium}
We write $V_n (\al,\be) = \sum_{\ell  < m} r_{n,\ell}^2  (\al,\be)$, where
\[
	r_{n,\ell} (\al,\be) = \E_n ( \| X \star \be_\ell \|^2 ) -  \E_n ( \| g_\al(Z) \star \be_\ell \|^2 )   . 
\]

\begin{prop}\label{propExistcne}
	Let $\ald \in \AL_n$, then 
	\[
		V_n ( \ald, \be ) = 0, \quad \forall \be \in \BE . 
	\]
	Moreover, $\forall \bed \in \BE$, $( \ald, \bed )$ is a consistent Nash equilibrium. 
\end{prop}

The proof is deferred to Appendix \ref{proof:propExistcne}. 
Unlike the previous two discriminators, the consistent Nash equilibrium
can be easily found among the generator $\al \in \AL_n$, and among all $\be \in \BE$. 
Indeed, due to Parseval's identity, this family of discriminator  
captures only the power spectrum of stationary processes
through the empirical expectations $ ( \E_n ( |\what{X} (\om)|^2 | )  )_{\om \in \Om_d} $. 
These expectations are the diagonal values of the matrix $F \Sigma_n F^\ast$ 
using the discrete Fourier transform $F$ on $\C^d$. 
Therefore the convolutional discriminator measures only a limited 
amount of information of $\Sigma_n$.

\paragraph{Uniqueness of consistent Nash equilibrium}

We next provide sufficient conditions for
any equilibrium of $V_n$ to be a consistent Nash equilibrium   
when $m=d$. It allows one to define a subset of
$\BE$ such that every equilibrium of the game is a consistent Nash equilibrium.

For $x \in \C^d$, we write $|x|^{\circ 2} (u) = |x(u)|^2$. 
We need the following result, 
\begin{lemm}\label{frame-beta}
	Assume $m=d$ and $ ( | \what { \be_\ell } |^{\circ 2} )_{ 0 \leq \ell < m}$ is a basis on $\R^d$, then $V_n (\al,\be) = 0$ is equivalent to $\al \in \AL_n$ almost surely. 
\end{lemm}

It is proved in Appendix \ref{app:frame-beta}.
Next we make the same assumption on $\alb$ as in Assumption \ref{assumX},
but the other technical assumptions regarding $n$ and $d$ are not needed.

\begin{theorem}\label{thm:consistenNE}
	Assume $(\ald,\bed)$ is an equilibrium  of $V_n$. If the following assumptions are satisfied,
	\begin{itemize}
		\item $\alb \not \in \AL_0$ and $m=d$,
		\item $ ( | \what { \bed_\ell } |^{\circ 2} )_{\ell < d}$ is a basis on $\R^d$,
		\item $\forall \ell < d, \bed_\ell \not \in \AL_0$,
	\end{itemize}
	then $\ald \in \AL_n$ almost surely and $(\ald,\bed)$ is a consistent Nash equilibrium.
\end{theorem}

The proof is deferred to Appendix \ref{proof:consistenNE}. 
To satisfy the conditions on the $\bed$ in Theorem \ref{thm:consistenNE}, 
we can consider a game using an open subset of $\BE$, e.g. 
\begin{equation}\label{eq:B0}
	 \BE_0 = \{ (\be_\ell)_{  \ell < m }  | \be_\ell \in \R^d ,  \mbox{ det}  (  
	 ( | \what { \be_\ell } |^{\circ 2} )_{\ell < m}     )\neq  0 , 
	 \mbox{ min}  (  (| \what { \be_\ell  } |^{\circ 2}  )_{\ell < m}    ) > 0 \} 	 	 
\end{equation}
By following the proof of Theorem \ref{thm:consistenNE}, one can verify that
\begin{corr}
If $\alb \not \in \AL_0$ and $m=d$, then
any equilibrium of the game $(\AL, \BE_0,V_n)$ is a consistent Nash equilibrium.
\end{corr}
Note that the set $ \BE_0 $ is not restrictive because 
Proposition \ref{propExistcne} shows that 
the choice of $\bed$ can be arbitrary once $\al \in \AL_n$. 
In the next section, we study numerically whether these conditions are implicitly satisfied 
along the dynamics of gradient-based methods to solve $(\AL, \BE,V_n)$.

\section{Numerical results}\label{sec:num}

In this section, we study gradient-descent-ascent 
optimization methods to find a consistent equilibrium of $V_n$. 
For the complex discriminator, 
we study the local stability of such methods
near the consistent non-Nash equilibrium in Theorem \ref{thm:FouriernonNash}.
When the sample size $n$ is large enough, 
we find that gradient-descent-ascent methods are nearly stable 
in the sense that the generator error remains almost constant. 
For the convolutional discriminator, we study their global convergence towards
consistent Nash equilibrium from random initialization.

\subsection{Real discriminator}

For the real discriminator, 
we provide further discussions 
about the non-existence of Nash equilibrium.
The proof of Theorem \ref{ref:realNonNE}
suggests that for any $\bed \in \BE$, 
one can expect to find an optimal 
$\al \in \AL$ such that $V_n(\al,\bed) = 0$. 
We verify that 
such solution can be computed,
and then we evaluate the
generator error $\| \Sigma - \Sigma_\al \|$
at these solutions. 

We consider  $\bed$ 
which maximizes $V_n (\al_0,\be) $,
with two different $\al_0$. 
One is a random $\al_0$, whose elements are sampled i.i.d from $\mathcal{N}(0,1/d)$. 
The other $\al_0$ is set to be the ground-truth parameter $\alb$. 
From the obtained $\bed$, we 
compute an optimal $\ald$ which solves $\min_\al V_n(\al,\bed) $
using gradient descent starting from $\al_0$. 
The optimal value $V_n(\ald,\bed) $ is expected to be very close to zero.
If the solution $\ald$ is very different from $\al_0$, then one should be able 
to detect such difference in the generator error.

Table \ref{tab:reald4} shows generator error differences
for the case $\alb = (1,0,0,0)$ and $d=4$. 
For the case of random $\al_0$,
we find that $ \| \Sigma - \Sigma_{\ald} \| < \| \Sigma - \Sigma_{\al_0} \|$, 
while for the case of $\al_0=\alb$, we find that 
$ \| \Sigma - \Sigma_{\ald} \| > \| \Sigma - \Sigma_{\al_0} \|$.
This result agrees with our theoretical analysis which shows that there is no Nash equilibrium using
the real discriminator. 
In the second case, 
$ \| \Sigma - \Sigma_{\al_0} \|=0$, thus 
the generator error difference measures directly 
$ \| \Sigma - \Sigma_{\ald} \| $. 
We find that $ \| \Sigma - \Sigma_{\ald} \| $ decreases 
with $n$ at a rate of $1/\sqrt{n}$. 
This is similar to the rate of the empirical estimator $ \| \Sigma - \Sigma_n \| $
computed later in Table \ref{tab:convErr}. 
The reason that in the second case, the 
generator error of $g_{\ald}$ is decreasing
is likely due to the initialization of the gradient-descent method (which is $\al_0$).
This is because according to \eqref{eq:aleq}, 
the solution set $\{ \al \in \AL: V_n ( \al, \bed) = 0 \}$ is likely 
to be very large, 
containing generators with both a small error and a large error. 

To compute $\max_{\be \in \BE} V_n (\al_0,\be) $,
we apply the classical power method \citep{golub2013matrix}
to the matrix $  ( \Sigma_n - \Sigma_{\al_0,n } )^2  $, 
according to Proposition \ref{prop1}.
In order to compute $\bed$ with a good precision, 
we use a total number of 
200 simulations, where $V_n$ is 
computed from independent random realizations of $X$ and $Z$. 
We keep only the first 100 simulations
where the relative difference
between $ V_n(\al_0,\bed) $ and 
$\| \Sigma_n - \Sigma_{\al_0,n} \|^2$ (computed from SVD) 
is smaller than $10^{-4}$. 
We then compute $\ald$ by 
the gradient descent method with a constant step-size. 
The method is run for at most 10000 iterations 
and it stops when the relative loss decrease is smaller than $10^{-8}$. 
We verify that all the optimal values $V_n(\ald, \bed)$ 
are very close to zero (around $10^{-10}$).

\begin{table}[t]
\normalsize
 \caption{\label{tab:reald4}
The generator error difference $\| \Sigma - \Sigma_{\ald} \|  - \| \Sigma - \Sigma_{\al_0} \| $ computed from the real discriminator using different sample size $n$ and $\al_0$. 
The mean and standard deviation of the difference are computed 
from multiple simulations. 
}
    \centering
    \begin{tabular}{ccc}
	& $\al_0$ random & \\ 
        \toprule
        $n$ & Mean &  Std. Dev \\
       \midrule \midrule
$10$ & -0.0656 & 2.0738 \\
$100$ & -0.8060 & 1.6804 \\
$1000$ & -0.6907 & 1.1686 \\
$10^4$ & -0.6143 & 0.9407 \\
$10^5$ & -0.6393 & 0.9723 \\
        \bottomrule
    \end{tabular}
    \begin{tabular}{ccc}
	& $\al_0 = \alb$ & \\
        \toprule
        $n$ & Mean &  Std. Dev \\
       \midrule \midrule
$10$ & 2.1073 & 1.8541 \\
$100$ & 0.5887 & 0.2476 \\
$1000$ & 0.1843 & 0.0611 \\
$10^4$ & 0.0568 & 0.0144 \\
$10^5$ & 0.0183 & 0.0053 \\
        \bottomrule
    \end{tabular}
\end{table}

\subsection{Complex discriminator}
Theorem \ref{thm:FouriernonNash} 
shows that there exists consistent non-Nash equilibrium in $V_n$.
We study how gradient-descent-ascent (GDA) methods \citep{gangd} 
behave near such equilibrium 
from both a discrete-time and continuous-time perspective.
The continuous-time GDA is interesting
as it is known that it can converge
towards a non-Nash equilibrium \citep{mazumdar2019finding}.

\paragraph{Discrete-time and continuous-time GDA}
For an initial $(\al^{(0)},\be^{(0)})$, 
the discrete-time GDA method iteratively updates $(\al,\be)$
by using the gradient vector of $V_n$. 
At iteration $t \geq 1$, it takes the form
\[
	\al^{(t)} =  \al^{(t-1)} - \eta \nabla_{\al} V_n ( \al^{(t-1)} , \be^{(t-1)} )  , \quad 
	\be^{(t)} =  \be^{(t-1)} + \eta \nabla_{\be} V_n ( \al^{(t-1)}, \be^{(t-1)} ) 
\]
where the step-size $\eta>0$. 
The GDA method in the discrete-time
differs from its continuous-time version, which is 
\[
	\frac{d \al^{(t)} } {dt} =  -   \nabla_{\al} V_n ( \al^{(t)}  , \be^{(t)}  )  , \quad 
	\frac{d \be^{(t)} } {dt}  =  \nabla_{\be} V_n ( \al^{(t)}, \be^{(t)} ) 
\]
For example, in bilinear games, where the symmetric part 
of the Jacobian matrix $J_n$ is zero, 
$(\al^{(t)}, \be^{(t)} )$ diverges from
the Nash equilibrium no matter how small $\eta$ is, 
whereas the continuous-time solution $( \al  , \be  ) $ 
turns around the equilibrium \citep{pmlr-v80-balduzzi18a}. 

The continuous-time GDA is simulated by the classical Runge-Kutta 4 
method as in \cite{qin2020training}, 
with the same step-size $\eta$ as the discrete-time GDA. 
In general, the step-size of $\al$ can be different 
from that of  $\be$ \citep{NIPS2017_twotimescale}. 

\paragraph{Local stability of GDA}
We study the stability of GDA near the non-Nash equilibrium $(\ald,\bed)$ in Theorem \ref{thm:FouriernonNash}. 
The initial point of GDA $(\al^{(0)},\be^{(0)})$ is taken to be a small perturbation of 
$(\ald,\bed) $ by an additive white Gaussian noise $\mathcal{ N} (0, \sigma^2 I_{d+2dm})$.
To measure the closeness of $(\al^{(t)},\be^{(t)})$ to $\AL_n$ and $\bed$ respectively, 
we compute
\[
	\epsilon_{\al}^{(t)} = \max_{\om \in \Om_d}  
	\bigg |  |  \what{\al^{(t)} } (\om) |^2 - \E_n (  |\what{X}(\om) |^2 )  /  \E_n (  |\what{Z}(\om) |^2 )  \bigg | , \quad
	\epsilon_{\be}^{(t)} = \sqrt {  \sum_\ell \| \be_\ell^{(t)}   - \bed_{\ell} \|^2   } .
\]

We consider the same case as above, $\alb = (1,0,0,0)$ and $d=4$, with $m=d$ features. 
We set $\ald$ to be a special element in the set $\AL_n$, 
and $\bed$ to be the discrete Fourier basis.\footnote{The special element is defined by $	\what{\ald} (\om) = \sqrt {    \E_n (  |\what{X}(\om) |^2 )  /  \E_n (  |\what{Z}(\om) |^2 ) } ,  \forall \om \in \Om_d $ . }
Figure \ref{fig:complexgda} shows how $\epsilon_{\al}^{(t)}$ and $\epsilon_{\be}^{(t)}$ evolves over 
the discrete-time GDA dynamics with $\eta=10^{-3} $ and $\sigma =10^{-3} $. 
We find that when $n=100$, GDA does not converge to the set $\AL_n$ in most simulations due 
to a large $\epsilon_{\al}^{(t)}$. We find that quite often it stopped when $V_n$ becomes too big (NaN). 
As $n \to \infty$, the algorithm becomes more stable.
When $n = 10000$, GDA seems to converge towards the set $\AL_n$ 
with a decreasing $\epsilon_{\al}^{(t)} $. 
It does not converge to $\bed$ as $\epsilon_{\be}^{(t)}$ does not change much 
from the initialization. 
This may be explained according to Remark \ref{rmq:uniform}, which shows that the gradient $\nabla_\be V_n$ tends to be very small at $\ald \in \AL_n$ as $n$ is big. 
This result shows that using a small number of iterations ($10^4$), 
the discrete-time GDA method is nearly stable around the set $\AL_n$ as long as $n$ is large enough. 

However, as we run the discrete-time GDA for $10^6$ iterations, 
we observe that even at $n=10000$, 
$V_n$ slowly increases with $t$ in most of the simulations. 
The unstable behavior of GDA suggests that 
the Jacobian matrix $J_n (\ald,\bed)$ has eigenvalues with negative real-part.  
In order to verify this, we 
simulate the continuous-time GDA for a large number of iterations ($10^6$) 
with various $\sigma \in \{  10^{-3} , 10^{-4} ,10^{-5}\}  $.
The result with a small $\sigma=10^{-5}$ and large $n=10000$ 
is given in Figure \ref{fig:convdynamicsLong} in Appendix \ref{app:num}. 
We find that in some simulations, the continuous-time GDA
also has a decreasing $V_n$ and then it remains constant over $t$. 
However in some simulations, it is unstable
due to a decreasing and then slowly increasing $V_n$.
It is remarkable that in most simulations, we find that
the generator error $ \|  \Sigma - \Sigma_{\al^{(t)}} \| $ remains constant over $t$.
Therefore the instability is mostly due to the discriminator
and it has a small impact on the generator quality.

  \begin{figure}[t]
      \centering
      \includegraphics[scale=0.33]{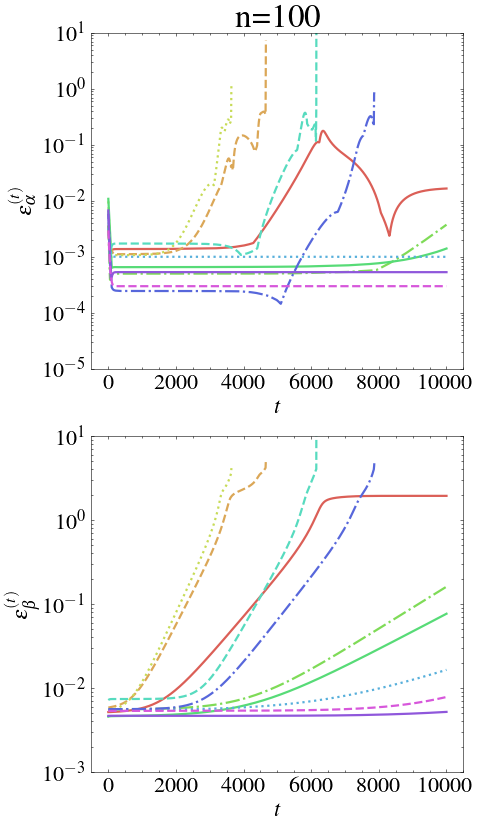} 
      \includegraphics[scale=0.33]{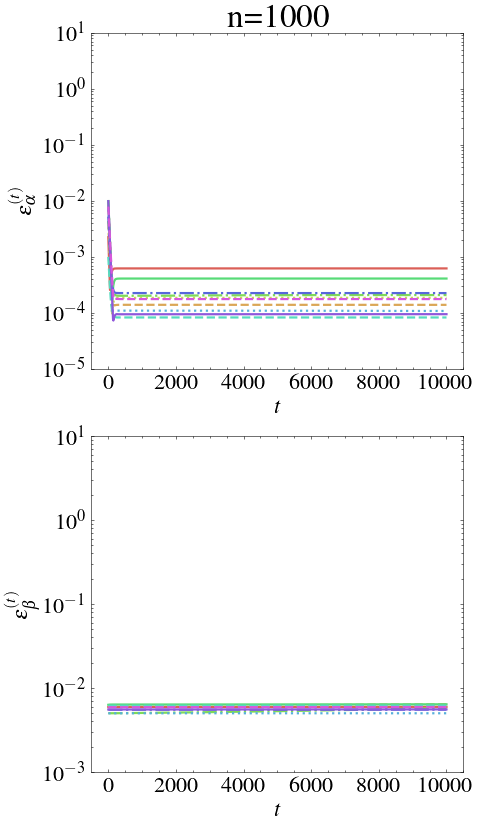} 
      \includegraphics[scale=0.33]{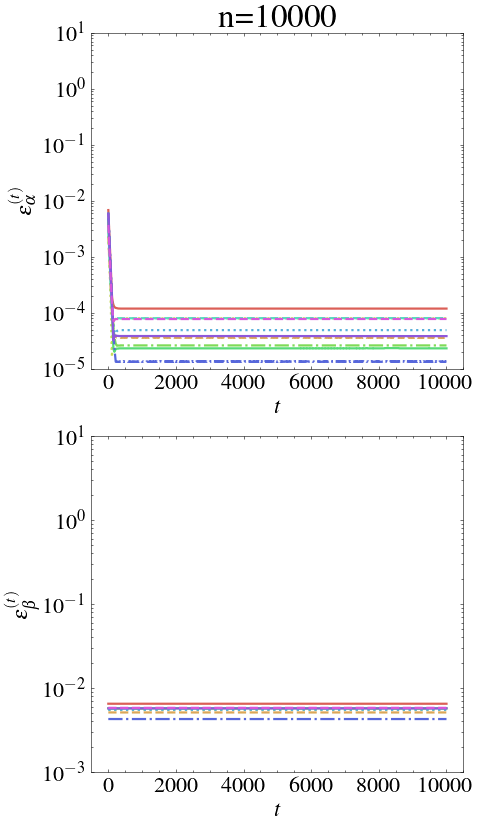} 
      \caption{
	The dynamics of $\epsilon_{\al}^{(t)}$ and $\epsilon_{\be}^{(t)}$
	as a function of the number of iterations ($10^4$ in total)
	of the discrete-time GDA on the complex discriminator. 
	Each of the 10 curves corresponds to a simulation from 
	independent samples of $X$ and $Z$. 
	The GDA stops if $V_n$ becomes nan. 
	}
      \label{fig:complexgda}
  \end{figure}

\subsection{Convolutional discriminator}

Theorem \ref{thm:consistenNE} provides 
sufficient conditions for an equilibrium of $V_n$
to be a consistent Nash equilibrium.
This motivates us to study whether such equilibrium
is the most likely limiting points of the continuous-time GDA method
starting from random initialization.

To study such global convergence, 
the initial point of GDA $(\al^{(0)},\be^{(0)})$
is sampled randomly from Gaussian white noise.\footnote{
Each element of $\al$ and $\be$ is sampled i.i.d from $\mathcal{N}(0,1/d)$.
Each vector $\be_\ell$ is further normalized to have a unit norm. 
}
To verify the sufficient conditions about $\be$ 
in Theorem \ref{thm:consistenNE}, 
we check that all the $\be^{(t)}$ belongs 
to the set $\BE_0$ defined in \eqref{eq:B0}. 
It amounts to compute the following quantities 
\[
	d_{\be}^{(t)} = | \mbox{ det}  (  | \what { \be^{(t)} } |^{\circ 2} ) | , \quad
	m_{\be}^{(t)} = \mbox{ min}  (  | \what { \be^{(t)} } |^{\circ 2} ) .
\]
By definition, $d_{\be}^{(t)} $ is the absolute value of the determinant of a matrix  
whose columns are formed by $| \what { \be_\ell^{(t)} } |^{\circ 2}$.
A non-zero  $d_{\be}^{(t)}$ implies that 
$ ( | \what {  \be^{(t)}_\ell } |^{\circ 2} )_{\ell < d}$ is a basis on $\R^d$. 
Similarly, if the minimal value of this matrix $ m_{\be}^{(t)}$ is non-zero
then $\forall \ell < d,  \be^{(t)}_\ell \not \in \AL_0$. 

The result for the case $\alb = (1,0,0,0)$ and $d=4$, with $m=d$ 
is given in Figure \ref{fig:convdynamics}. 
We observe that the conditions of Theorem \ref{thm:consistenNE} are
well respected in all the simulations. 
We also see that in almost all simulations, 
the continuous-time GDA converges to small values of $V_n$ and $\epsilon_{\al}^{(t)} $.
When $n \leq 1000$, there are several simulations which 
have a slow convergence. The slowness 
seems to be related to the smallest value of 
$d_{\be}^{(t)} $ over $t$. 
To verify the convergence of the continuous-time GDA, 
a longer run at $n=10000$ with $10^6$ iterations is given 
in Figure \ref{fig:convdynamicsLong} in Appendix \ref{app:num}. 
We find that the GDA always has a convergent $V_n$. 
This suggests that the algorithm has found a consistent Nash equilibrium
$(\al_n,\be_n)$ in the set $ \AL_n \times \BE_0$. 


Some values of $\epsilon_{\al}^{(t)} $ in Figure \ref{fig:convdynamics}
are quite close to zero, suggesting that 
at the last iteration of GDA, the generator is close to $\AL_n$. 
As a consequence, the generator error should be close 
to the empirical error $ \| \Sigma - \Sigma_n \|$. 
This is because when $n$ is large, 
$\Sigma_n $ is very close to $\Sigma_{\ald,n}$ for 
$\ald \in \AL_n$, according to \eqref{eq:asympSigma}. 
To validate this, we report the generator error 
at the last iteration of GDA in Table \ref{tab:convErr},
estimated from 100 simulations.
This error is compared to the empirical estimator of $\Sigma = \E( X X^\intercal)$ 
using the same number of samples of $X$. 
The decreasing mean and standard deviation as $n$ grows
show that the GDA converges well to the generators 
whose errors are close to those of $\AL_n$. 
We also find that when $n$ is large, the generator error (mean) is 
slightly smaller on average than the empirical estimator.
This is not so surprising as $\Sigma_{\al^{(T)}}  $ 
is a Toeplitz and circulant covariance matrix, as $\Sigma$. 
But the empirical covariance $\Sigma_n$ does not satisfy these properties. 
However, the standard deviation of the generator error is slightly larger 
than the empirical estimator. 
This should be related to the global convergence of GDA. 
Indeed, when $n$ is small ($n=10$), 
we find that in a few simulations (four among a hundred), 
the GDA diverges to NaN in $V_n$ (these cases are excluded in the error computation). 
No such divergence happens when $n \geq 100$.
In this case, the larger standard deviation 
may be related to the fluctuations of the limiting values of $\epsilon_{\al}^{(t)}$.

  \begin{figure}
      \centering
      \includegraphics[scale=0.33]{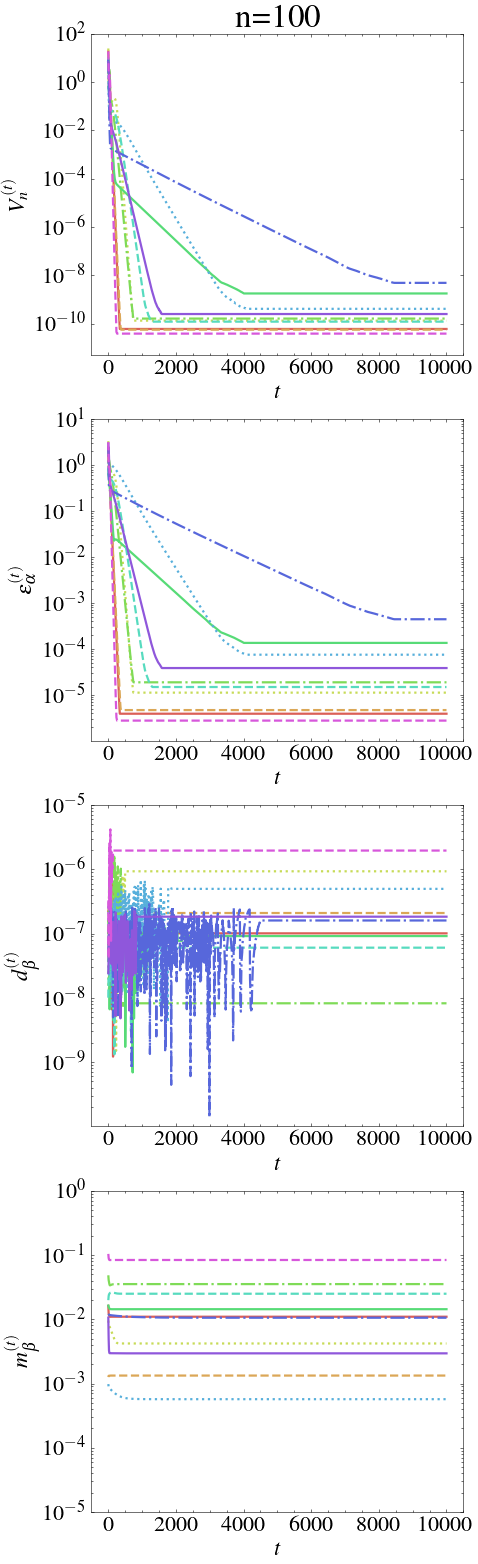} 
    \includegraphics[scale=0.33]{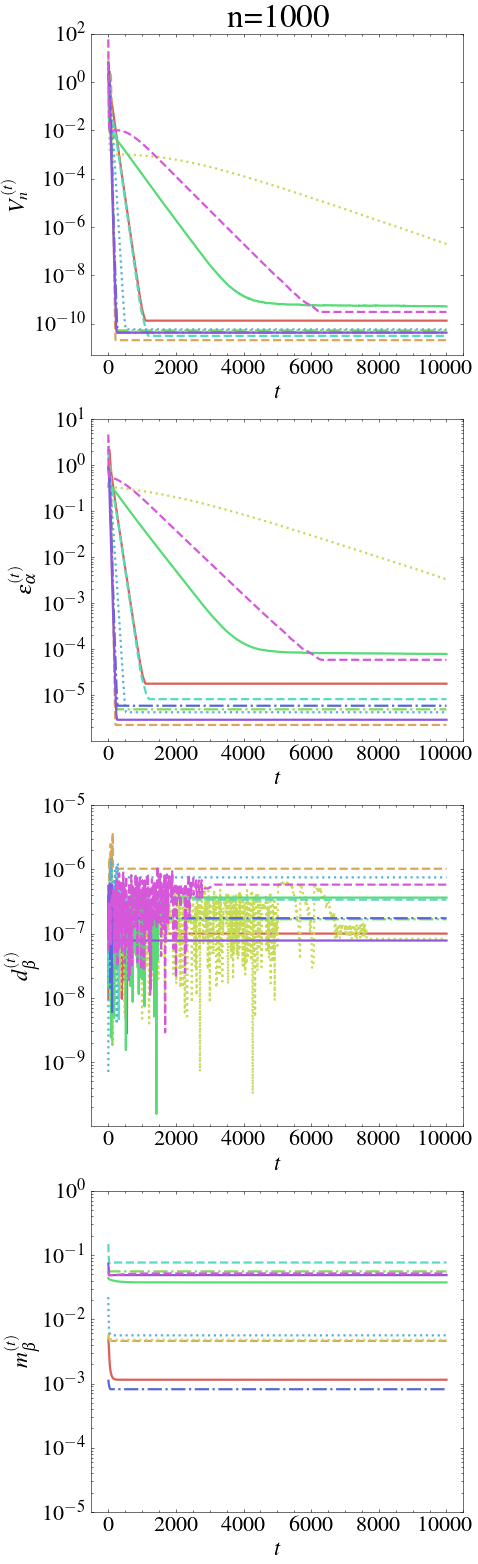} 
      \includegraphics[scale=0.33]{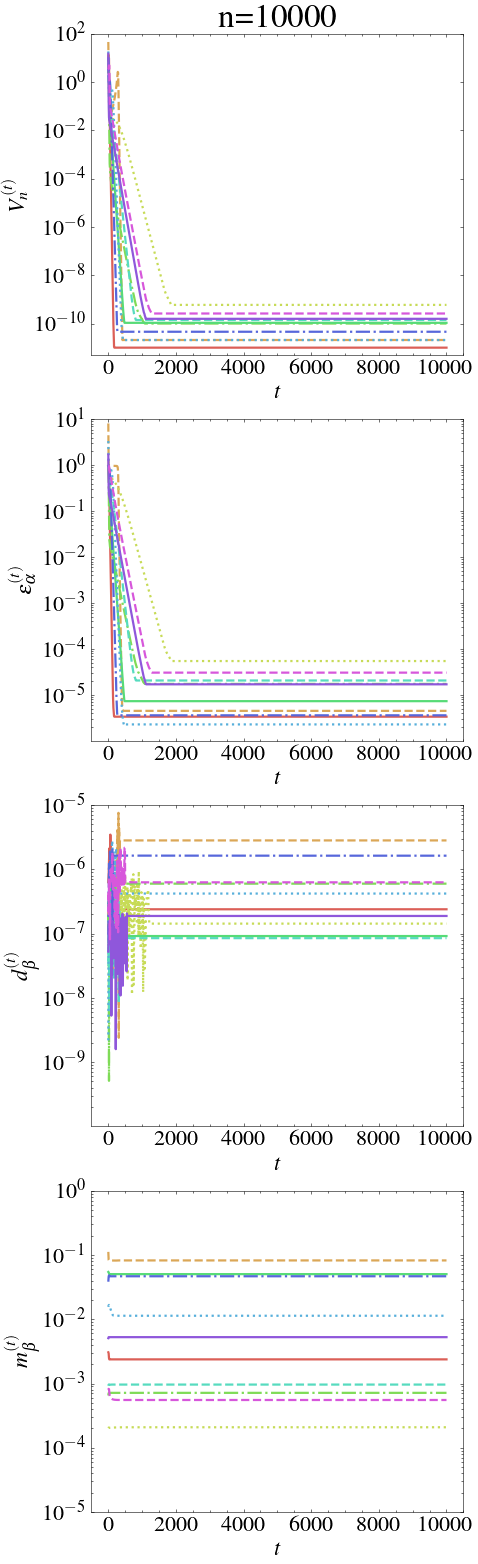} 
      \caption{
	The convergence of the continuous-time GDA as a function of number of iterations ($10^4$ in total) on the convolutional discriminator. 
	From top to bottom: 
	$V_n^{(t)} = V_n(\al^{(t)},\be^{(t)})$, $\epsilon_{\al}^{(t)}$,  $d_{\be}^{(t)}$ and $m_{\be}^{(t)}$. 
	From left to right:  $n \in \{100,1000,10000\}$. Each of the 10 curves corresponds to one simulation. 	}
      \label{fig:convdynamics}
  \end{figure}

\begin{table}
	\normalsize
	\caption{\label{tab:convErr}
		Evaluation of the generator computed by 
		the continuous-time GDA on the convolutional discriminator.
		Left: the generator error $\| \Sigma - \Sigma_{\al^{(T)}}   \| $ with $T=10^4$.
		Right: the empirical error $\| \Sigma - \Sigma_n   \| $. 
		The mean and standard deviation of these errors are computed 
		from multiple simulations and for various sample size $n$. 
	}
	\centering
	\begin{tabular}{ccc}
		& $\| \Sigma - \Sigma_{\al^{(T)}}   \| $  & \\ 
		\toprule
		$n$ & Mean &  Std. Dev \\
		\midrule \midrule
		$10$ &  0.9705 & 1.1089 \\
		$100$ &  0.2874 & 0.1583 \\
		$1000$ & 0.0790 & 0.0352 \\
		$10^4$ & 0.0233 & 0.0106 \\
		$10^5$ & 0.0089 & 0.0097 \\
		\bottomrule
	\end{tabular}
	\begin{tabular}{ccc}
		& $\| \Sigma - \Sigma_n   \| $ & \\
		\toprule
		$n$ & Mean &  Std. Dev \\
		\midrule \midrule
		$10$ & 1.0446 & 0.4060 \\
		$100$ & 0.3450 & 0.0965 \\
		$1000$ & 0.1051 & 0.0234 \\
		$10^4$ & 0.0338 & 0.0076 \\
		$10^5$ & 0.0104 & 0.0026 \\
		\bottomrule
	\end{tabular}
\end{table}

\section{Conclusion}\label{sec:conclude}

In this paper, we take a moment matching and simultaneous game perspective 
to study the existence of Nash equilibrium in the finite-sample realizable setting.
By focusing on a particular generator family of Gaussian stationary processes,
we show that GANs have a rich variability of equilibrium properties. 
Although these properties vary greatly with the choice 
of the discriminator family,
our results 
suggest that a suitable discriminator family can 
result in the existence of consistent Nash or non-Nash equilibrium 
which are both meaningful solutions.
We also find that GDA methods have nearly stable
local convergence or global convergence properties
when the number of samples is large enough. 

To extend our work to non-Gaussian distributions 
remains a challenging problem, as one would have
to construct discriminator families beyond second-order moments.  
In order to extract non-Gaussian information, 
one example is to use the rectifier non-linearity to 
define a discriminator family 
so as to match higher-order moments between two distributions \citep{li2021making}. 
As we can see from the convolutional discriminator, 
the existence of consistent Nash equilibrium depends on a 
careful balance between the discriminator family and the generator family.
As a consequence, symmetry properties of a distribution also need to be considered 
in order to capture invariant non-Gaussian information \citep{ZHANG2021199}.
The difficulty is to define what it means in general to achieve a balance.
Recent advances in designing GANs 
to avoid the curse of dimensionality 
in high dimensional density estimation \citep{bai2018approximability,JMLR:v22:20-911,Feizi2020} 
could provide some insights into this problem. 
On the other hand, the existence of consistent non-Nash equilibrium 
seems to be more relevant to practical GANs. In our case, 
its existence depends on a perfect moment-matching condition (i.e. $V_n ( \ald, \bed )=0$ in 
Theorem \ref{thm:FouriernonNash}) which allows one to extract 
sufficient information of stationary Gaussian processes. 
The sufficiency is important for GDA to be stable along $\al$, 
as it constraints the solution set 
$\{ \al \in \AL, V_n ( \al, \bed )=0 \}$ to consistent generators. 
Whether such condition holds 
for other moments remains open.  
\acks{This work is supported by 3IA Artificial and Natural Intelligence
Toulouse Institute, French "Investing for the Future - PIA3" program
under the Grant agreement ANR-19-PI3A-0004, and by 
Toulouse INP ETI and AAP unique CNRS 2022 under the Project LRGMD.}

\bibliography{references}

\appendix

\section{Proof of Proposition \ref{prop1}}\label{proofkkt}

\begin{proof}
	We apply the KKT necessary condition \citep[Theorem 12.1]{nocedal2006numerical} to the following Lagrangian function for $ \be  \in \BE$ and $\la \geq 0$, 
	\[
	L( \be, \la) = -  ( \be^\intercal (\Sigma_n - \Sigma_{\al,n}) \be )^2 - \lambda ( 1 - \| \be \|^2 ) 
	\]
	If $\bed$ is an optimal solution of $V_n$, then there exists $\lad$, such that:
	\begin{equation}\label{eq:kkt}
	\nabla_\be L ( \bed,\lad  ) = 0 , \quad 
	\lad ( 1 - \| \bed \|^2) = 0, \quad 
	\| \bed \|^2 \leq 1, \quad
	\lad \geq 0 .
	\end{equation}
	Let $A = \Sigma_n - \Sigma_{\al,n}$, then
	$	\nabla_\be L ( \bed,\lad  ) =   -4 (  \be^\intercal A \be  ) A \be + 2 \la \be   $. 
	From \eqref{eq:kkt}, we have two situations:
	\begin{itemize}
		\item Case $\lad > 0$ and $ \| \bed \| = 1 $: 
		$\nabla_\be L ( \bed,\lad  ) =0$ implies that $ \lad \bed   =2  (  (\bed)^\intercal A \bed  ) A \bed  $. 
		This means that 
		$\bed $ is an eigenvector of $A$, whose eigenvalue 
		$\mu = \frac{  \lad } {   2 (  (\bed)^\intercal A \bed  ) } $ is non-zero
		because $(\bed)^\intercal A \bed $ mush be strictly positive. 
		Therefore the value $V_n (\bed,\lad) = \mu^2 > 0 $. Note that this situation can not happen if $A = 0$. 
		\item Case $ \lad  = 0 $:  $\nabla_\be L ( \bed,\lad  )  = 0$ implies that $ (  (\bed)^\intercal A \bed  ) A \bed = 0$. Therefore either $(\bed)^\intercal A \bed  =  0$ or $A \bed = 0$, and $V_n(\bed,\lad) = (  (\bed)^\intercal A \bed  )^2 = 0$.
	\end{itemize}
	 When $A \neq 0$, the above analysis shows that the optimal solution of $V_n$ is attained at $\bed$. Moreover, it is a unit-norm eigenvector of $ A $ which has the maximal absolute eigenvalue so that $\mu^2$ is maximal. The maximal value of $V_n$ thus coincides with $\| \Sigma_n - \Sigma_{\al,n } \|^2  $. 	
\end{proof}

\section{Proof of Lemma \ref{lem:nonzero}}\label{proofnonzero}
\begin{proof}
	For any $\al \in \AL$, 
	we show that the event $ B_1 = \{ \| \Sigma_n - \Sigma_{\al,n} \| = 0 \} $
has a zero probability under Assumption \ref{assumX}.
	Indeed, under $B_1$, $\Sigma_n = \Sigma_{\al,n}$ and therefore $\forall (u,u')$, 
	\[
		\frac{1}{n} \sum_{i=1}^n \alb \star \bar{z}_i (u) \alb \star \bar{z}_i (u')  = 
		\frac{1}{n} \sum_{i=1}^n \al \star z_i (u) \al \star z_i (u') . 
	\]
	where $\{ z_i \}$ and $\{ \bar{z}_i \}$ are the $n$ i.i.d samples of $Z$ and $\bar{Z}$.  
	
	Applying the Fourier transform along both $u$ and $u'$, we have equivalently $\forall (\om,\om') \in \Om_d^2$
	\[
		\frac{1}{n} \sum_{i=1}^n \what{\alb} (\om) \what{\bar{z}}_i (\om) 
								 \what{\alb} (\om')^\ast \what{\bar{z}}_i (\om')^\ast   = 
		\frac{1}{n} \sum_{i=1}^n \what{\al} (\om) \what{z}_i (\om) 
								\what{\al} (\om')^\ast \what{z}_i (\om')^\ast 								 
	\]
	
	As $d $ is even, we take $\om=\om'=0$ or $\om=\om'=\pi$, and it follows that 
	\begin{align} \label{eq:z0}
		|  \what{\alb} (0) |^2  \frac{1}{n} \sum_{i=1}^n  | \what{\bar{z}}_i (0) |^2  &=  
		 | \what{\al} (0) |^2   \frac{1}{n} \sum_{i=1}^n | \what{z}_i (0) |^2   \\ 
		|  \what{\alb} ( \pi ) |^2  \frac{1}{n} \sum_{i=1}^n  | \what{\bar{z}}_i (  \pi  ) |^2  &=  | \what{\al} (  \pi  ) |^2
		\frac{1}{n} \sum_{i=1}^n  | \what{z}_i (  \pi  ) |^2  \label{eq:z1}
	\end{align}
	
	Taking $\om = 0$ and $\om'=\pi$, we have 
	\begin{align}\label{eq:z2}
		 \what{\alb} (0)  \what{\alb} (\pi)  \frac{1}{n} \sum_{i=1}^n   \what{\bar{z}}_i (0) \what{\bar{z}}_i (\pi)  =  
		 \what{\al} (0)  \what{\al} (\pi)  \frac{1}{n} \sum_{i=1}^n   \what{z}_i (0) \what{z}_i (\pi) 
	\end{align}
	
	To arrange the above terms \eqref{eq:z0},\eqref{eq:z1},\eqref{eq:z2}, 
	we need the following event to avoid division by zero
	\begin{align*}
		B_2 =  \left \{ \frac{1}{n} \sum_{i=1}^n  | \what{\bar{z}}_i (0) |^2 \neq 0 , \, 	  
		 \frac{1}{n} \sum_{i=1}^n  | \what{z}_i (0) |^2 \neq 0, \,
		 \frac{1}{n} \sum_{i=1}^n  | \what{\bar{z}}_i (\pi) |^2 \neq 0,\,
		 \frac{1}{n} \sum_{i=1}^n  | \what{z}_i (\pi) |^2 \neq 0 \right \}
	\end{align*}
	
	As $\alb \not \in \AL_0$, we have $\what{\alb}(0) \neq  0$ and $\what{\alb}(\pi) \neq  0$.
	Then under $B_1\cap B_2$, it follows from \eqref{eq:z0},\eqref{eq:z1} and \eqref{eq:z2} that the following event holds, 
	\[
		B_3 = \left \{ \frac{  | \frac{1}{n} \sum_{i=1}^n   \what{\bar{z}}_i (0) \what{\bar{z}}_i (\pi)   |^2  } {   \frac{1}{n} \sum_{i=1}^n  | \what{\bar{z}}_i (0) |^2  \frac{1}{n} \sum_{i=1}^n  | \what{\bar{z}}_i (  \pi  ) |^2  }  = 
		\frac{  | \frac{1}{n} \sum_{i=1}^n   \what{z}_i (0) \what{z}_i (\pi)   |^2  } {   \frac{1}{n} \sum_{i=1}^n  | \what{z}_i (0) |^2  \frac{1}{n} \sum_{i=1}^n  | \what{z}_i (  \pi  ) |^2  }   
		\right \}
	\]
	This means that $ B_1 \cap B_2 \subset  B_3$. 
	
	We claim that 
	\begin{equation}\label{eq:events}
		\mbox{Prob} ( B_2) = 1 , \quad 
		\mbox{Prob} ( B_3) = 0
	\end{equation}
	From \eqref{eq:events}, the statement of this lemma holds, 
	because we will have $ \mbox{Prob} (B_1)  =  0$
	by using the fact that $\mbox{Prob} (B_1 \cap B_2)  \leq \mbox{Prob} ( B_3)=0$ and $ \mbox{Prob} (B_1 ) = \mbox{Prob} (B_1 \cap B_2) -  \mbox{Prob} (B_2)  + \mbox{Prob} (B_1 \cup B_2)  = \mbox{Prob} (B_1 \cap B_2) $. 
	
	To show \eqref{eq:events}, 
	we denote $ y (\om) = ( \what{z}_i (\om) )_{i \leq n} \in \C^n$, then 
	both $ y (0)$ and $y(\pi)$ follow $\mathcal{N}(0, \sigma^2 I_n)$ for $\sigma>0$. 
	Moreover $ y (0)$ is independent of $y(\pi)$ since for any $i \leq n$, $ \what{z}_i (0) $ is independent of $ \what{z}_i ( \pi  ) $. Similarly for $ \bar{y} (\om) = ( \what{\bar{z}}_i (\om) )_{i \leq n} \in \C^n$ at $\om=0$ or $\om=\pi$. 
	Therefore we have $	\mbox{Prob} ( B_2) = 1 $ because 
	the distribution of $ \| y (0) \|^2 $ (and $\| y (\pi) \|^2$, $\| \bar{y} (0) \|^2$, $\| \bar{y} (\pi) \|^2$) is Chi-square of degree $n$ (up to some constant normalization). 
	
	To show $\mbox{Prob} ( B_3) = 0$, 
	we notice that the event $B_3$ is equivalent to 
	\begin{equation}\label{eq:joint}
		\frac{  | \lb \bar{y}(0) ,  \bar{y}( \pi ) \rb |^2  } {  
		\|  \bar{y}(0)  \|^2  \|  \bar{y}( \pi  )  \|^2 }  = 
		\frac{  | \lb y(0) ,  y( \pi ) \rb |^2  } {  
		\|  y(0)  \|^2  \|  y( \pi  )  \|^2 }
	\end{equation}
	We next argue that for $n \geq 2$, the distribution of LHS or RHS of \eqref{eq:joint} has a continuous density on $[0,1]$. As the LHS and RHS 
	are sampled from the same distribution independently, the chance that they are the same is thus zero. To verify this, it is sufficient to derive the distribution of 
	\[
		\frac{   \lb Y_0 ,  Y_1 \rb   } {  
			\|  Y_0  \|   \|  Y_1   \| }
	\]
	where $Y_0 \sim \mathcal{N}(0,I_n)$ and $Y_1 \sim \mathcal{N}(0,I_n)$ are independent Gaussian white noise in $\R^{n}$. 
	
	Assume $F(\theta) = \mbox{Prob} \left(  \frac{   \lb Y_0 ,  Y_1 \rb   } {  
		\|  Y_0  \|   \|  Y_1   \| } \leq \sin(\theta) \right)$ for $\theta \in [-\pi/2,\pi/2]$ . To compute $F(\theta)$, we denote the coordinates of $Y_0 \in \R^n$ by $x=(x_1,\cdots,x_n)$. Then by the rotational invariant property of the distribution of $Y_0$ and $Y_1$, we fix $Y_1 = (1,0,\cdots,0)$ and obtain 
	\[
		F(\theta) = \frac{1}{ (\sqrt{2 \pi})^n }   \int_{  x_1 \leq \sin(\theta) \| x \| }  
		e^{- \| x \|^2  / 2} dx 
	\]
	
	We first focus on the case $\theta \in [0,\pi/2]$. 
	The condition $ x_1 \leq \sin(\theta) \| x \| $ is equivalent to 
	\[
		x_1 \leq \sqrt{  \frac{ \sin^2(\theta) }{ 1 - \sin^2(\theta) }  ( x_2^2 + \cdots + x_n^2  )  }   
	\]
	Denote the cumulative function of the standard normal distribution by $\Phi$. By writing 
	$(x_2,\cdots x_n)$ in the spherical coordinate for $r > 0, \psi_2 \in [0,\pi], \cdots \psi_{n-2} \in [0,\pi]$, and $\psi_{n-1} \in [0,2 \pi]$, 
	\[
		\left (r \cos(\psi_2),r \sin(\psi_2) \cos( \psi_3),\cdots, r \sin(\psi_2) \cdots \sin(\psi_{n-2}  )\sin( \psi_{n-1} ) \right), 
	\]
	we have that 
	\begin{align*}
		F(\theta) &= \frac{1}{ (\sqrt{2 \pi})^n }  \int e^{ - (x_2^2 + \cdots + x_n^2) / 2 }  \int_{-\infty}^{ \tan(\theta) \sqrt{x_2^2 + \cdots + x_n^2 } } 
		e^{ - x_1^2 / 2} d x_1  d x_2 \cdots d x_n 	\\ 
		&= \frac{1}{ (\sqrt{2 \pi})^{n-1} }  \int e^{ - (x_2^2 + \cdots + x_n^2) / 2 } 
		\Phi \left(   \tan(\theta) \sqrt{x_2^2 + \cdots + x_n^2 }  \right)  
		 d x_2 \cdots d x_n 	 \\ 
		 &=
		 \frac{1}{ (\sqrt{2 \pi})^{n-1} }  \int e^{ - r^2 / 2 } 
		 \Phi \left(   \tan(\theta) r  \right)  
		 r^{n-2} \sin^{n-3} ( \psi_2) \cdots \sin( \psi_{n-2} ) dr d \psi_2  \cdots d \psi_{n-1}	\\	
		 &=  c_n \int_{0}^\infty  	
		  e^{ - r^2 / 2 } 
		 \Phi \left(   \tan(\theta) r  \right)  
		 r^{n-2} dr 
	\end{align*}
	with a normalization constant $c_n$. It follows that the density
	\begin{align*}
	F'(\theta) &= 	c_n 
	\int_{0}^\infty  	
	e^{ - r^2 / 2 } 
	\Phi' \left(   \tan(\theta) r  \right)  \frac{r}{ \cos ^2(\theta) }
	r^{n-2} dr 
	 \\ 	  	
	 &= \frac{c_n }{ \sqrt{2 \pi} } 
	 \int_{0}^\infty  	
	 e^ {   - ( 1 + \tan^2(\theta)  ) r^2 /2  } \frac{r^{n-1} }{ \cos^2 (\theta) }
	  dr 	 \\
	  &\propto \frac{1 }{  \cos^2 (\theta)  } 
	  \int_{0}^\infty  	
	  e^ {   - ( 1 + \tan^2(\theta)  ) t /2  } t^{ (n-2)/2 } dt 
	\end{align*}
	This integral can be computed from the Gamma distribution which gives for $\al > 0, \be > 0$, 
	\[
		\int_0^\infty t^{ \al - 1} e^{ - \be t }  dt = \Gamma( \al ) / \be^\al  . 
	\]	
	Taking $\al = n / 2 $, and $\be = (1 + \tan^2(\theta) )/2$, we conclude that for $\theta \in [0,\pi/2]$, 
	\begin{equation}\label{eq:densityF}
		F'(\theta) \propto \frac{ 1 } { \cos^2(\theta) }   \frac{ \Gamma(n/2 ) } {  ( \frac{ 1 + \tan^2(\theta) }{2} )^{ n /2 } }   \propto \cos^{n-2} (\theta) .
	\end{equation}
	As the distribution of $\lb Y_0, Y_1 \rb$ is symmetric around zero, the density $F'(\theta)$ is a symmetric function around $\theta=0$. Thus \eqref{eq:densityF} holds also for $\theta \in [-\pi/2 , 0]$. As $\sin(\theta)$ is differentiable and monotone increasing on $[-\pi/2,\pi/2]$, a change of variable shows that the density of $  \frac{ \lb Y_0 ,  Y_1 \rb   } {  	\|  Y_0  \|   \|  Y_1 \| } $ exists and it is supported on $[-1,1]$. 	
\end{proof}

\section{Proof of Remark \ref{rmq:uniform}}\label{app:uniform}

We show that 
\[
  	\sup_{  \ald \in \AL_n  }  \| 	\Sigma_n - \Sigma_{ \ald , n}    \|  \to  0 , \quad n \to \infty, \quad \mbox{in probability}
\]

\begin{proof}
The key idea is to show that $\AL_n $ is a bounded set with high probablity when $n$ is large enough, and to use classfical results about the convergence of emprical covariance matrices through a uniform upper bound of $\| 	\Sigma_n - \Sigma_{ \ald , n}    \|$.

Using the trianglar inequality of the norm $\| \cdot \|$, we have 
\begin{equation}\label{eq:trian}
	\sup_{  \ald \in \AL_n  }  \|  \Sigma_{ \ald , n } - \Sigma_n \|   \leq 
\sup_{  \ald \in \AL_n  }   \|  \Sigma_{ \ald, n } - \Sigma_{\ald} \| + \sup_{  \ald \in \AL_n  } \| \Sigma_{\ald} - \Sigma \| +  \| \Sigma - \Sigma_n  \| 
\end{equation}

It is sufficient to show that each of the three terms on the RHS of \eqref{eq:trian} converges to zero in probablity. Since  $X$  is a Gasussian distribution on $\R^d$, the convergence of $ \| \Sigma - \Sigma_n  \| $ follows immediately from the classfical results established for sub-Gaussian distributions \cite[Theorem 4.7.1]{vershynin_2018}.

Recall that 
\[
	\AL_n =  \{  \al \in \R^d | | \what{\al}(\om) |^2 = \E_n (  |\what{X}(\om) |^2 )  /  \E_n (  |\what{Z}(\om) |^2 ), \forall \om \in \Om_d \} . 
\]

Proposition \ref{propSigmaFourier} implies that, $\forall \ald \in \AL_n$
\begin{align*}
	  \| \Sigma_{\ald} - \Sigma \|
&= \max_{\om \in \Om_d} |  |\what{\ald}  (\om) |^2 - | \what{\alb} (\om) |^2  |   \\ 
&= 
\max_{\om \in \Om_d}  | \E_n (  |\what{X}(\om) |^2 )  /  \E_n (  |\what{Z}(\om) |^2 ) - | \what{\alb} (\om) |^2  | \\
&= 
\max_{\om \in \Om_d}  |  (\E_n (  |\what{\Zb}(\om) |^2 )  /  \E_n (  |\what{Z}(\om) |^2 ) -  1) | \what{\alb} (\om) |^2  |
\end{align*}
Applying the law of large numbers to the $n$ samples of $\Zb$ and the $n$ samples of $Z$, we have 
\[
	\sup_{  \ald \in \AL_n  } \| \Sigma_{\ald} - \Sigma \|    \to  0 , \quad n \to \infty , \quad  \mbox{in probability}
\]

For the convergence of $\sup_{  \ald \in \AL_n  }   \|  \Sigma_{ \ald, n } - \Sigma_{\ald} \| $, we note that for any $\al \in \AL$, $\forall (u,u')$, 
\[
	 ( \Sigma_{ \al, n }  - \Sigma_{\al}  ) (u,u') =  \sum_{v,v'}  \al(u-v) \al(u'-v') (  \Sigma_n^0 (v,v') - \Sigma^0 (v,v') ) , 
\]
where $   \Sigma_n^0 (v,v')  = \E_n ( Z(v) Z(v') )$ and $   \Sigma^0 (v,v') =  \E ( Z(v) Z(v') )$. 
Let $\al_u (v) = \al(u-v) $, then 
\[
	 ( \Sigma_{ \al, n }  - \Sigma_{\al}  ) (u,u') = \al_u^\intercal  ( \Sigma_n^0  - \Sigma^0 ) \al_{u '}
\]

Denote $ M_\al = d \| \al \|^2$, we next check that 
\begin{equation}\label{eq:upperM}
	\| \Sigma_{ \al, n }  - \Sigma_{\al}  \|^2  \leq M_\al^2 \|  \Sigma_n^0  - \Sigma^0 \|^2 
\end{equation}

To verify \eqref{eq:upperM}, we use the Cauchy-Schwartz inequality, 
\begin{align*}
	  \| \Sigma_{ \al, n }  - \Sigma_{\al}  \|^2  
&= \sup_{ \| w \| \leq 1 }   \|   ( \Sigma_{ \al, n }  - \Sigma_{\al} ) w \|^2 \\ 
&= \sup_{ \| w \| \leq 1 } \sum_{u} |  \sum_{u'}  \al_u^\intercal  ( \Sigma_n^0  - \Sigma^0 ) \al_{u'}   w(u')  |^2  \\
&\leq \sup_{ \| w \| \leq 1 } \sum_{u}  \|      \al_u \|^2   \cdot \|  \sum_{u'} ( \Sigma_n^0  - \Sigma^0 ) \al_{u'}   w(u')  \|^2   \\ 
&\leq \sup_{ \| w \| \leq 1 } M_\al   \|  \Sigma_n^0  - \Sigma^0 \|^2 \|  \sum_{u'}  \al_{u'}   w(u')  \|^2   \\ 
&= M_\al   \|  \Sigma_n^0  - \Sigma^0 \|^2 \sup_{ \| w \| \leq 1 }   \sum_{v'} |  \sum_{u'}  \al_{u'}   (v') w(u')  |^2  \\
&\leq M_\al^2   \|  \Sigma_n^0  - \Sigma^0 \|^2 . 
\end{align*}

From \eqref{eq:upperM}, we conclude that 
\[
	\sup_{  \ald \in \AL_n  } \| \Sigma_{\ald} - \Sigma \|   \leq  ( \sup_{\ald \in \AL_n } M_{\ald} ) \|  \Sigma_n^0  - \Sigma^0 \|	 
\]

For any $\ald \in \AL_n$, 
\[
	\| \ald \|^2 = d^{-1} \sum_{\om \in \Om_d} | \what{\ald} (\om) |^2     = d^{-1}  \sum_{\om   \in \Om_d }  \frac{ \E_n (  |\what{\Zb}(\om) |^2 )  } {  \E_n (  |\what{Z}(\om) |^2 )  }    |  \what{\alb}   (\om ) |^2   
\]
Therefore $M_{\ald} = d \| \ald \|^2$ does not vary in the set $\AL_n$. 

According to \cite[Theorem 4.7.1]{vershynin_2018}, 
$\|  \Sigma_n^0  - \Sigma^0 \| $ converges to zero in probability when $n \to \infty$. 
Therefore, it remains to show that $\exists C>0$, $\forall \delta>0$, $\exists N \in \Z$ such that if $n \geq N$, then 
\[
	\mbox{ Prob}  (  \sup_{\ald \in \AL_n}  M_{\ald}  \leq C ) > 1 - \delta. 
\]
This is true because the law of large numbers, applied to $Z$ and $\Zb$, implies that 
\[
	\max_{ \om \in \Om_d } \left |  \frac{ \E_n (  |\what{\Zb}(\om) |^2 )  } {  \E_n (  |\what{Z}(\om) |^2 )  }   - 1 \right |  \to 0    , \quad n \to \infty, \quad \mbox{in probability}, 
\]
i.e. $ \sup_{\ald \in \AL_n} M_{\ald}  \to d \| \alb \|^2$ as $n \to \infty$ in probability.

\end{proof}

%
%
%
%
%
%
%
%
%
%
%
%
%
%
%

\section{Proof of Proposition \ref{propExistcne}}\label{proof:propExistcne}

\begin{proof}
	For $ \ald \in \AL_n$, we are going to show that $r_{n,\ell} ( \ald, \be) =0$ 
	for all $\ell$. This implies that $V_n  ( \ald, \be ) = \sum_{\ell}  r_{n,\ell}^2  ( \ald, \be )  = 0$
	for $\forall \be \in \BE$. 
	Indeed, using Parseval's identity, $\forall x \in \R^d$
	\[
	\|  x \star \be_\ell  \|^2 =  d^{-1} \sum_{\om \in \Om_d} |\what{x} (\om)|^2 |  \what{\be}_\ell (\om)  |^2 
	\]
	This implies that for $\forall \ell < m$, $\forall \be \in \BE$, 
	\[
	r_{n,\ell} ( \ald, \be) \propto \sum_\om 
	\left ( \E_n (   |\what{X} (\om)|^2  )
	- \E_n (   |\what{\ald} (\om) \what{Z} (\om) |^2  )  
	\right ) 
	|  \what{\be}_\ell (\om)  |^2 = 0
	\]
	To show that $	(\ald , \bed )$ is a Nash equilibrium for any $  \bed \in \BE$, we verify that 
	for any $\al \in \AL$ and $\be \in \BE$, 
	\[
	V_n ( \ald , \be )  \leq  V_n ( \ald, \bed ) \leq  V_n (\al,\bed) 
	\]
	The consistency of  $( \ald, \bed )$ follows from Proposition \ref{propConsistent}. 

	%
	%
	
\end{proof}

\section{Proof of Lemma \ref{frame-beta}}
\label{app:frame-beta}

\begin{proof}
	If $\al \in \AL_n$, then Proposition \ref{propExistcne} implies that $V_n (\al,\be) = 0$ for any $\be \in \BE$. 
	We next show that if $( | \what {\be}_\ell |^{\circ 2} )_{ \ell < d }$ is a basis, then $V_n (\al,\be) = 0$ implies that $\al \in \AL_n$ almost surely. 
	
	Using Parseval's identity, we write
	\begin{align*}
		r_{n,\ell} (\al,\be) &=  ( \E_n ( \| X \star \be_\ell \|^2 ) -  \E_n ( \| g_\al(Z) \star \be_\ell \|^2 )   )   \\
		&= d^{-1} (  \sum_\om \E_n (  | \what{X}(\om)   \what{\be_\ell} (\om) |^2 ) -   \sum_\om \E_n  ( | \what{g_\al(Z)}(\om)   \what{\be_\ell} (\om) |^2 )   ) \\
		&= d^{-1} \sum_\om  (   \E_n ( | \what{X}(\om) |^2  ) - \E_n ( | \what{\al}(\om) \what{Z}(\om) |^2  ) ) |\what{\be_\ell} (\om) |^2 \\
		&= \lb h ,  |\what{\be_\ell} |^{\circ 2} \rb, 
	\end{align*}
	where $ h(\om) =  d^{-1} (  \E_n ( | \what{X}(\om) |^2  ) - \E_n ( | \what{\al}(\om) \what{Z}(\om) |^2  )   ) $ for $\om \in \Om_d$. 
	
	By definition, $V_n (\al,\be) = 0$ is equivalent to $ r_{n,\ell} (\al,\be)  = 0$ for all $\ell $, i.e. 
	\[
	\lb h , | \what {\be}_\ell |^{\circ 2} \rb = 0, \quad  \forall \ell < d
	\]
	It implies that $h = 0$ because $ | \what {\be}_\ell |^{\circ 2}$ is a basis  of $\R^d$. Therefore $\al \in \AL_n$ almost surely. 
\end{proof}

\section{Proof of Theorem \ref{thm:consistenNE}}\label{proof:consistenNE}

\begin{proof}

	Assume $(\ald,\bed)$ is an equilibrium of $V_n$. 
	Consider two cases, 
	\begin{itemize}
		\item $\ald \not \in \AL_0$: we show that $\ald \in \AL_n$, and therefore by Proposition \ref{propExistcne}, $(\ald,\bed)$ is a consistent Nash equilibrium. 	
		\item $\ald \in \AL_0$: we show that $(\ald,\bed)$ can not be an equilibrium, because $\nabla_\be V_n (\ald,\bed) \neq 0$. 
	\end{itemize}

	\textbf{Case $\ald \not \in \AL_0$}: From Lemma \ref{frame-beta}, it is 
	sufficient to check that $V_n (\ald,\bed)  = 0$ almost surely. 
	
	The gradient of $V_n$ with respect to $\al$ and $\be$ are computed as in \eqref{eq:gradsVr},
	with 
	\begin{align}
		\nabla_{\al} r_{n,\ell} (\al,\be) &= -2  ( \be_\ell \star \tilde{\be}_\ell ) \star \E_n ( Z \star \tilde{Z} ) \star \al   \label{eq:gdcompal}\\
		\nabla_{\be_\ell} r_{n,\ell} (\al,\be) &= 2   (   \E_n (  X \star \tilde{X} \star \be_\ell ) -  \E_n (  g_\al( Z)  \star \widetilde{ g_\al(Z) }  \star \be_\ell   ) )  \label{eq:gdcompbe}
	\end{align}
	
	As $\nabla_\al V_n (\ald,\bed) = 0$, \eqref{eq:gdcompal} implies that 
	\begin{equation}\label{eq:LHD}
	\sum_\ell r_{n,\ell}  (\ald,\bed)  ( \bed_\ell \star \tilde{\bed_\ell} ) \star \E_n ( Z \star \tilde{Z} ) \star \ald   = 0
	\end{equation}
	We take the Fourier transform on the LHD of \eqref{eq:LHD}. It results in an equivalent equation
	\[
	\sum_\ell r_{n,\ell}   (\ald,\bed)   |  \what{ \bed_\ell} (\om) |^2    \E_n ( | \what{Z}(\om)  |^2 )   \what{\ald}  (\om)  = 0, \quad \forall \om \in \Om_d
	\]	
	Since for any $\om \in \Om_d$, $  \E_n ( | \what{Z}(\om)  |^2 )   \what{\ald}  (\om)  \neq 0$ almost surely, it follows that almost surely
	\begin{equation}\label{eq:sum0}
	\sum_\ell r_{n,\ell}    (\ald,\bed)  |  \what{ \bed_\ell} (\om) |^2 = 0 , \quad \forall \om \in \Om_d . 
	\end{equation}
	
	As $ ( | \what { \bed_\ell } |^{\circ 2} )_{\ell < d}$ is a basis on $\R^d$, \eqref{eq:sum0} implies that 
	$r_{n,\ell} (\ald,\bed) = 0$ for all $\ell $, i.e. $V_n (\ald,\bed) = 0$. 
	
	\textbf{Case $\ald  \in \AL_0$}: by the definition of $\AL_0$, there exists $\om_0 \in \Om_d $ such that $\what{\ald} (\om_0) = 0$. 
	
	Firstly, we have almost surely
	\[
	\AL_n  \cap \AL_0 =  \emptyset
	\]
	because $\al \in \AL_n$ implies that almost surely
	\[
	|\what{\al}(\om) |^2 = | \what{\alb}(\om) |^2 \E_n (  |\what{\bar{Z}}(\om) |^2 )  /  \E_n (  |\what{Z}(\om) |^2 ) \neq 0, \forall \om  \in \Om_d
	\]
	
	As $\ald \in \AL_0$, this implies that almost surely $\ald \not \in \AL_n$
	Then from Lemma \ref{frame-beta}, $V_n (\ald,\bed) \neq 0$ almost surely,
	so that there exists $\ell_0 < d$ such that $r_{n,\ell_0} (\ald,\bed)  \neq 0$. 
	We next verify that the gradient of $V_n$ with respect to $\be_{\ell_0}$ is non-zero.
	From  \eqref{eq:gradsVr} and \eqref{eq:gdcompbe}
	\begin{align}
		\nabla_{\be_{\ell_0}} V_n (\ald,\bed)  &= 2 r_{n,{\ell_0}} (\ald,\bed)  \nabla_{\be_{\ell_0} } r_{n,\ell_0} (\ald,\bed)  \nonumber \\
		&= r_{n,{\ell_0}} (\ald,\bed)  		(   \E_n (  X \star \tilde{X} \star \bed_{\ell_0} ) -  \E_n (  g_{\ald}( Z)  \star \widetilde{ g_{\ald} (Z) }  \star \bed_{\ell_0}   ) ) \label{eq:gradBEcomp}
	\end{align}
	
	As $r_{n,\ell_0} (\ald,\bed)  \neq 0$, taking the Fourier transform of \eqref{eq:gradBEcomp}, $ \nabla_{\be_{\ell_0}} V_n (\ald,\bed) = 0 $ implies that 
	\[
	\E_n ( | \what{X} (\om) |^2 )  \what{\be}_{\ell_0} (\om)  =
	\E_n ( | \what{g_{\ald}(Z)} (\om) |^2 )   \what{\be}_{\ell_0} (\om)   , \forall \om \in \Om_d
	\]
	
	As $\be_{\ell_0} \not \in \AL_0$, it implies that almost surely $   \E_n ( | \what{X} (\om_0) |^2 )  = \E_n ( | \what{g_{\ald} (Z)} (\om_0) |^2 )  =  0 $ because $\what{\ald} (\om_0) = 0$. This is contradictory because  $\alb \not \in \AL_0$ implies that $   \E_n ( | \what{X} (\om_0) |^2 ) $ is almost surely non-zero.

	%
	%
	%
	%
	%
	%
	%
	%
	
\end{proof}

\section{Additional numerical results}\label{app:num}


\begin{figure}
	\centering
	\includegraphics[scale=0.33]{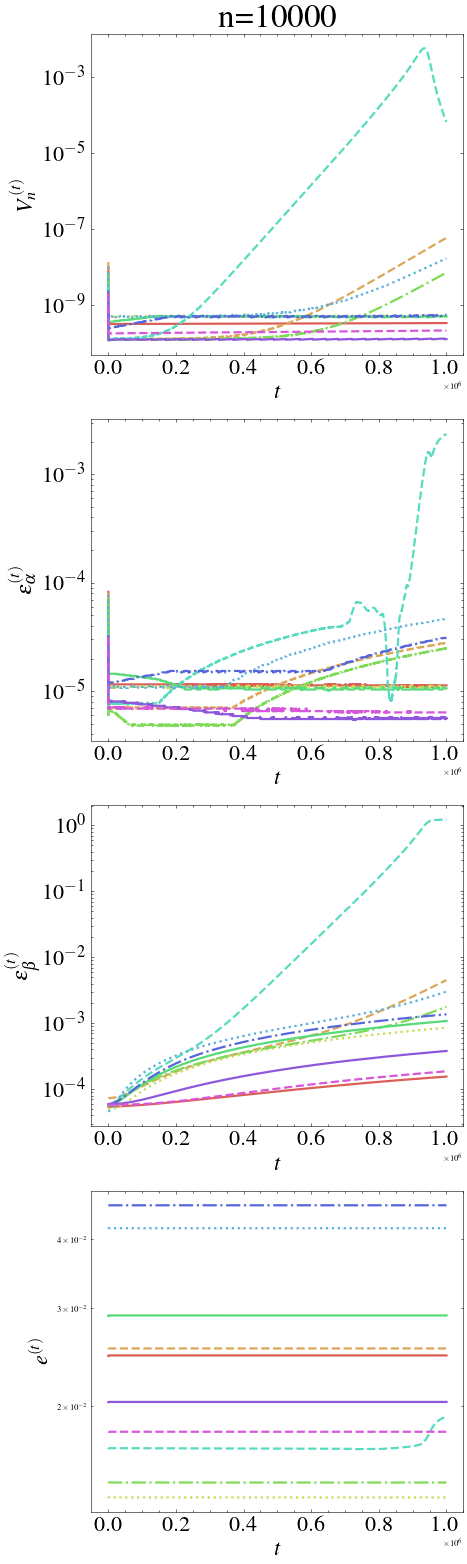} 	
	\includegraphics[scale=0.33]{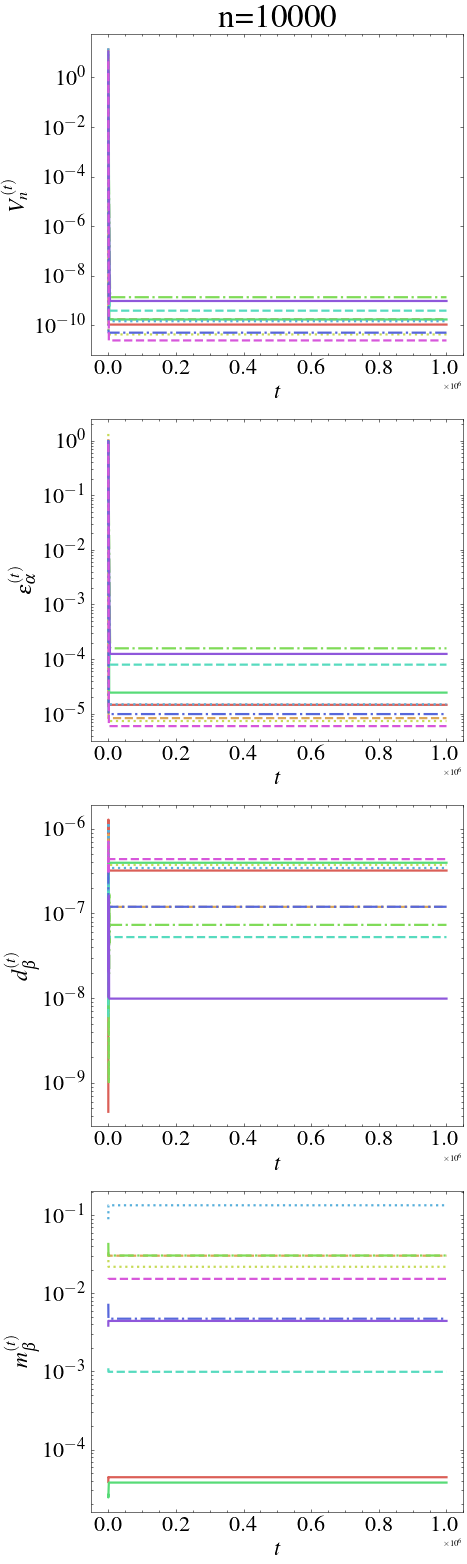} 
	\caption{
		The local stability (left) and global convergence (right) of the continuous-time GDA as a function of the number of iterations ($10^6$ in total). 
		Left: the dynamics of $V_n^{(t)} = V_n(\al^{(t)},\be^{(t)})$, $\epsilon_{\al}^{(t)}$,  $\epsilon_{\be}^{(t)}$ and $e^{(t)} = \| \Sigma - \Sigma_{ \al^{(t)}} \|$ on the complex discriminator.
		Right: the dynamics of $V_n^{(t)} = V_n(\al^{(t)},\be^{(t)})$, $\epsilon_{\al}^{(t)}$,  $d_{\be}^{(t)}$ and $m_{\be}^{(t)}$ on the convolutional discriminator. 
	}
	\label{fig:convdynamicsLong}
\end{figure}



\end{document}